\documentclass{article}
\usepackage{uai2019}
\usepackage[margin=1in]{geometry}

\usepackage{times}

\usepackage{microtype}
\usepackage{graphicx}
\usepackage{subfigure}
\usepackage{booktabs} 

\usepackage{hyperref}
\hypersetup{hidelinks}

\usepackage{natbib}

\usepackage{amsmath}
\usepackage{amssymb}
\usepackage{amsthm}

\newtheorem{lemma}{Lemma}

\title{Bayesian Optimization with Binary Auxiliary Information}

\author{Yehong Zhang, Zhongxiang Dai, and Bryan Kian Hsiang Low\\
Department of Computer Science, National University of Singapore, Republic of Singapore\\
\{yehong, daizhongxiang, lowkh\}@comp.nus.edu.sg}

\begin{document}
	
\maketitle

\begin{abstract}\vspace{-0.2mm}
This paper presents novel mixed-type \emph{Bayesian optimization} (BO) algorithms to accelerate the optimization of a target objective function by exploiting correlated auxiliary information of binary type that can be more cheaply obtained, such as in policy search for reinforcement learning and  hyperparameter tuning of machine learning models with early stopping. To achieve this, we first propose a mixed-type \emph{multi-output Gaussian process} (MOGP) to jointly model the continuous target function and binary auxiliary functions. Then, we propose information-based acquisition functions 
such as \emph{mixed-type entropy search} (MT-ES) and \emph{mixed-type predictive ES} (MT-PES) for mixed-type BO based on the MOGP predictive belief of the target and auxiliary functions. The exact acquisition functions of MT-ES and MT-PES cannot be computed in closed form and need to be approximated. We derive an efficient approximation of MT-PES via a novel mixed-type random features approximation of the MOGP model whose cross-correlation structure between the target and auxiliary functions can be exploited for improving the belief of the global target maximizer using observations from evaluating these functions.
We propose new practical constraints to relate the global target maximizer to the binary auxiliary functions.
We empirically evaluate the performance of MT-ES and MT-PES with synthetic and real-world experiments.
\end{abstract} 

\section{INTRODUCTION}\label{intro}
\emph{Bayesian optimization} (BO) has recently demonstrated with notable success to be highly effective in optimizing an unknown (possibly noisy, non-convex, and/or with no closed-form expression/derivative) target function using a finite budget of often expensive function evaluations~\citep{shahriari16}. 
 As an example, BO is used by~\citet{Snoek2012} to determine the setting of input hyperparameters (e.g., learning rate, batch size of data) of a \emph{machine learning} (ML) model that maximize its validation accuracy (i.e., output of the unknown target function).
Conventionally, a BO algorithm relies on some choice of acquisition function (e.g., improvement-based \citep{shahriari16} such as probability of improvement or \emph{expected improvement} (EI) over currently found maximum, information-based \citep{villemonteix09} such as \emph{entropy search} (ES)~\citep{hennig2012} and \emph{predictive entropy search} (PES)~\citep{hernandez2014}, or \emph{upper confidence bound} (UCB) \citep{srinivas10}) as a heuristic to guide its search for the global target maximizer. 
To do this, the BO algorithm exploits the chosen acquisition function
to repeatedly select an input for evaluating the unknown target function that trades off between sampling at or near to a likely target maximizer based on a \emph{Gaussian process} (GP) belief of the unknown target function (exploitation) vs. improving the GP belief (exploration) until the budget is expended.

In practice, the expensive-to-evaluate target function often correlates well with some cheaper-to-evaluate \emph{binary} auxiliary function(s) that delineate the input regions potentially containing the global target maximizer and can thus be exploited to boost the BO performance.
%
For example, automatically tuning the hyperparameters of a sophisticated ML model (e.g., deep neural network) with BO is usually time-consuming as it may incur several hours to days to evaluate the validation accuracy of the ML model under each selected hyperparameter setting when training with a massive dataset.
To accelerate this process, consider an auxiliary function whose output 
is a binary decision of whether the validation accuracy of the ML model under the selected input hyperparameter setting will exceed a pre-specified threshold,
which is recommended by some early/optimal stopping mechanism~\citep{muller2007} after a small number of training epochs.
Such auxiliary information of binary type is cheaper to obtain and can quickly delineate the input regions containing the best hyperparameter setting, hence incurring less time for exploration.
%
%
Similarly, to find the best reinforcement learning policy for an AI agent in a game or a real robot in a task with binary outcomes (e.g., success or failure)~\citep{tesch13},
maximizing the success rate (i.e., the unknown target function with a continuous output type) averaged over multiple random environments can be accelerated by deciding whether the selected setting of input policy parameters is promising in a single environment (i.e., the auxiliary function with a binary output type).
To search for the optimal setting of a system via user interaction~\citep{shahriari16}, gathering implicit/binary user feedback (e.g., click or not, like or dislike) is often easier than asking for an explicit rating/ranking of a shown example.
The above practical examples motivate the need to design and develop a \emph{mixed-type} BO algorithm that can naturally trade off between exploitation vs.~exploration over the target function with a \emph{continuous} output type and the cheaper-to-evaluate auxiliary function(s) with a \emph{binary} output type
for finding or improving the belief of the global target maximizer, which is the focus of our work here.

In this paper, we generalize information-based acquisition functions like ES and PES to \emph{mixed-type ES} (MT-ES) and \emph{mixed-type PES} (MT-PES) for mixed-type BO (Section~\ref{acqfn}).
To the best of our knowledge, these are the first BO algorithms that exploit correlated \emph{binary} auxiliary information for accelerating the optimization of a continuous target objective function.
%
Different from \emph{continuous} auxiliary functions which have been exploited by a number of multi-fidelity BO algorithms \citep{Huang2006,Swersky2013,kandasamy2016,kandasamy2017,poloczek2017,sen2018}, the binary auxiliary functions in our problem make the widely used Gaussian likelihood inappropriate and prevent a direct application of existing multi-fidelity BO algorithms.\footnote{We discuss other related works in Appendix~\ref{related}.}

To resolve this, we first propose a mixed-type multi-output GP to jointly model the unknown continuous target function and binary auxiliary functions. 
Although the exact acquisition function of MT-PES cannot be computed in closed form, the main contribution of our work here is to show that it is in fact possible to derive an efficient approximation of MT-PES via (a) a novel \emph{mixed-type random features} (MT-RF) approximation of the MOGP model whose cross-correlation structure between the target and auxiliary functions can be exploited for improving the belief of the global target maximizer using the observations from evaluating these functions (Section~\ref{MORF}), and (b) new practical constraints relating the global target maximizer to the binary auxiliary functions (Section~\ref{approxPES}).
We empirically evaluate the performance of MT-ES and MT-PES with synthetic and real-world experiments (Section~\ref{experiment}).\vspace{-1.5mm} 
%
%
%
%
\section{PROBLEM SETUP}\label{setup}\vspace{-1.5mm}
In this work, we have access to an unknown target objective function $f_1$ and $M-1$ auxiliary functions $f_2, \ldots, f_M$ defined over a bounded input domain $D \subset \mathbb{R}^d$ such that each input $x \in D$ is associated with a noisy output $y_i(x)$ for $i = 1, \ldots, M$. 
As mentioned in Section~\ref{intro}, a cost $\lambda_i(x)$ is incurred to evaluate function $f_i$ at each input $x \in D$ and the target function is more costly to evaluate than the auxiliary functions, i.e., $\lambda_1(x) > \lambda_i(x)$ for $i = 2, \ldots, M$.
Then, the objective is to find the global target maximizer $
x_* \triangleq \arg\max_{x \in D} f_1(x)
$
with a lower cost by exploiting the cheaper auxiliary function evaluations, as compared to evaluating only the target function.
Our problem differs from that of the conventional multi-fidelity BO in that only the target function returns continuous outputs (i.e., $y_1(x) \in \mathbb{R}$) while the auxiliary functions return binary outputs (i.e., $y_i(x) \in \{1, -1\}$ for $i = 2, \ldots, M$).\vspace{-1.6mm} 
%
\section{MIXED-TYPE MULTI-OUTPUT GP} 
\label{CMOGP}\vspace{-1.5mm}
Various types of multi-output GP models \citep{Cressie1993,Wackernagel98,Webster01,Skolidis2012, Bonilla2007,Teh2005,Alvarez2011} have be used to jointly model target and auxiliary functions with continuous outputs. However, none of them can be used straightforwardly in our problem to model the mixed output types due to the non-Gaussian likelihood $p(y_i(x)|f_i(x))$ of the auxiliary functions.
To resolve this issue, we generalize the \emph{convolved multi-output Gaussian process} (CMOGP) to model the correlated functions with mixed continuous and binary output types by approximating the non-Gaussian likelihood using \emph{expectation propagation} (EP), as discussed later. The CMOGP model is chosen for generalization due to its convolutional structure which can be exploited for deriving an efficient approximation of our acquisition function, as described in Section~\ref{MFPES}.

Let the target and auxiliary functions $f_1, \ldots, f_M$ be jointly modeled as a CMOGP 
which defines each function $f_i$ as a convolution between a smoothing kernel $K_i$ and a latent function\footnote{To ease exposition, we consider a single latent function. Note, however, multiple latent functions can be used to improve the modeling~\citep{Alvarez2011}. More importantly, our proposed MT-RF approximation and MT-PES algorithm can be easily generalized to handle multiple latent functions, as shown in Appendix~\ref{multi_L}.}
$L$ with an additive bias $m_i$:\vspace{-1mm}
\begin{equation}\label{func}
f_i(x) \triangleq m_i + \int_{x'\in D}K_i( x-x')\ L(x')\ \text{d}x'\ .
\end{equation}
Let $D_i^+ \triangleq \{\langle x, i\rangle\}_{ x \in D}$ and $D^+ \triangleq \bigcup_{i=1}^M D_i^+$. As shown by~\citet{Alvarez2011}, if $\{L( x)\}_{x\in D}$ is a GP, then $\{f_i(x)\}_{\langle x,i\rangle \in D^+}$ is also a GP, that is, every finite subset of $\{f_i(x)\}_{\langle x,i\rangle \in D^+}$ follows a multivariate Gaussian distribution. Such a GP is fully specified by its \emph{prior} mean
$\mu_i(x) \triangleq \mathbb{E}[f_i(x)]$
and covariance
$\sigma_{ij}(x, x') \triangleq \text{cov}[f_i(x), f_j(x')]$
for all $\langle x, i\rangle, \langle x', j\rangle\in D^+$, the latter of which characterizes both the correlation structure within each function (i.e., $i=j$) and the cross-correlation between different functions (i.e., $i \neq j$).
Specifically, let $\{L( x)\}_{x\in D}$ be a GP with zero mean, prior covariance $\sigma_{xx'} \triangleq \mathcal{N}( x - x'| \underline{0}, \Gamma^{-1})$, and
$
K_i( x) \triangleq \sigma_{s_i}\mathcal{N}( x|\underline{0}, P^{-1}_i)
$
where $\sigma^2_{s_i}$ is the signal variance controlling the intensity of the outputs of  $f_i(x)$, $\Gamma$ and $P_i$ are diagonal precision matrices controlling, respectively, the degrees of correlation between outputs of latent function $L(x)$
and cross-correlation between outputs of $L(x)$ and $f_i(x)$.
Then, $\mu_i(x) = m_i$ and
\begin{equation}\label{kernel}
\sigma_{ij}(x, x') = \sigma_{s_i}\sigma_{s_j}\mathcal{N}( x -  x'| \underline{0}, \Gamma^{-1}\hspace{-0.5mm}+P^{-1}_i\hspace{-0.5mm}+P^{-1}_j)\ .
\end{equation}
In this work, we assume the Gaussian and probit likelihoods for the target and auxiliary functions, respectively:\vspace{-1mm}
\begin{equation}\label{l}
\begin{array}{rcl}
 p(y_1(x)|f_1(x)) &\hspace{-2.4mm}\triangleq &\hspace{-2.4mm} \mathcal{N}(f_1(x), \sigma^2_{n_1})\ , \vspace{0.5mm}\\
p(y_i(x)|f_i(x)) &\hspace{-2.4mm}\triangleq &\hspace{-2.4mm}\Phi_\text{cdf}(y_i(x)f_i(x)) 
\end{array}
\end{equation}
for $i= 2, \ldots, M$.
Supposing a column vector $y_{X} \triangleq (y_i(x))^\top_{\langle x, i\rangle \in X}$ of outputs are observed by evaluating each $i$-th function $f_i$
at a set $X_i \subset D_i^+$ of input tuples where $X \triangleq \bigcup_{i=1}^M X_i$, the predictive belief/distribution of $f_Z  \triangleq (f_i(x))^\top_{\langle x,i\rangle \in Z}$ for any set $Z \subseteq D^+ \setminus X$ of input tuples can be computed by 
\begin{equation}\label{fz}
p(f_Z |  y_X) = \int p(f_Z|f_X)\ p(f_X | y_X)\ \text{d}f_{X}\ .
\end{equation}
For conventional CMOGP with only continuous output types, \eqref{fz} can be computed analytically since both $p(f_Z|f_X)$ and $p(f_X | y_X)$ are Gaussians \citep{Alvarez2011}. Unfortunately, the non-Gaussian likelihood in \eqref{l} makes the integral in \eqref{fz} intractable. To resolve this issue, the work of~\citet{pourmohamad2016} has proposed a sampling strategy based on a sequential Monte Carlo algorithm which, however, is computationally inefficient and makes the approximation of our proposed acquisition function (Section~\ref{MFPES}) prohibitively expensive.
In contrast, we approximate the non-Gaussian likelihood using EP to derive an analytical approximation of \eqref{fz}, as detailed later. EP will be further exploited in Section~\ref{MFPES} for approximating our proposed acquisition function efficiently.\vspace{-1.2mm}
\subsection{MIXED-TYPE CMOGP PREDICTIVE INFERENCE}\vspace{-1mm}
%
Let $X_B \triangleq \bigcup_{i=2}^M X_i$ be a set of input tuples of the auxiliary functions. The posterior distribution $p(f_X|y_X)$ in \eqref{fz} can be computed by
\begin{equation}\label{pfX}
\hspace{-1.7mm}
\begin{array}{l}
p(f_{X_1}, f_{X_B} |  y_{X_1}, y_{X_B}) \vspace{0.5mm}\\
 \propto p(f_{X_1}, f_{X_B})\ p(y_{X_1} | f_{X_1})\ p(y_{X_B} | f_{X_B}) \vspace{0.5mm}\\
\displaystyle = p(f_{X_1}|f_{X_B}) p(f_{X_B}) p(y_{X_1} | f_{X_1}) \hspace{-1.5mm} \prod_{\langle x, i\rangle \in X_B} \hspace{-1.5mm} p(y_i(x) | f_i(x)) \\
= p(f_{X_1}|f_{X_B})\ p(y_{X_1} | f_{X_1})\ q(f_{X_B})
\end{array}
\end{equation}
where $q(f_{X_B}) \triangleq p(f_{X_B}) \prod_{\langle x, i\rangle \in X_B}  p(y_i(x) | f_i(x))$ can be approximated with a multivariate Gaussian $\mathcal{N}(f_{X_B}|\tilde{\mu}_B, \tilde{\Sigma}_B)$ using EP by approximating each non-Gaussian likelihood as a Gaussian. Let
\begin{equation}\label{EP}
\begin{array}{rcl}
p(y_i(x) | f_i(x)) &\hspace{-2.4mm}= &\hspace{-2.4mm}\Phi_\text{cdf}(y_i(x)f_i(x)) \vspace{1mm}\\
&\hspace{-2.4mm} \approx&\hspace{-2.4mm} \tilde{Z}_i(x)\ \mathcal{N}(f_i(x)|\tilde{\mu}_i(x), \tilde{\sigma}^2_i(x))
\end{array}
\end{equation}
for all $\langle x, i\rangle \in X_B$.
Following the EP procedure in Section $3.6$ of \citet{Rasmussen2006}, the parameters $\tilde{\mu}_i(x)$ and $\tilde{\sigma}^2_i(x)$ can be computed analytically and 
\begin{equation}\label{qfB}
\begin{array}{rcl}
\tilde{\mu}_B &\hspace{-2.4mm}= &\hspace{-2.4mm} \Sigma_{X_BX_B}(\tilde{\Sigma}^{-1}\tilde{\mu} + \Sigma_{X_BX_B}^{-1}\mu_{X_B})\vspace{1mm} \\
\tilde{\Sigma}_B &\hspace{-2.4mm}= &\hspace{-2.4mm} (\tilde{\Sigma}^{-1} + \Sigma_{X_BX_B}^{-1})^{-1}
\end{array}
\end{equation}
where $\tilde{\mu} \triangleq (\tilde{\mu}_i(x))^\top_{\langle x, i\rangle \in X_B}$, $\tilde{\Sigma}$ is a diagonal matrix with diagonal components $\tilde{\sigma}^2_i(x)$ for $\langle x, i\rangle \in X_B$, $\Sigma_{AA'} \triangleq (\sigma_{ij}(x, x'))_{\langle x, i\rangle \in A,\langle x', j\rangle \in A'}$, and $\mu_A \triangleq (\mu_i(x))^\top_{\langle x, i\rangle \in A}$ for any $A,A'\subseteq D^+$. 

By combining \eqref{qfB}, \eqref{pfX}, and \eqref{l} with \eqref{fz} (Appendix~\ref{a_pred}), the predictive belief $p(f_Z|y_X)$ can be approximated by a multivariate Gaussian $\mathcal{N}(\mu_{Z|X},\Sigma_{ZZ|X})$ with the following \emph{posterior} mean vector and covariance matrix:
\begin{equation} \label{var}
\begin{array}{rcl}
\mu_{Z|X} & \hspace{-2.4mm}\triangleq & \hspace{-2.4mm}\mu_Z + \Sigma_{ZX} \Lambda^{-1}(\tilde{y}_X - \mu_X) \vspace{1mm}\\
\Sigma_{ZZ|X} & \hspace{-2.4mm}\triangleq& \hspace{-2.4mm} \Sigma_{ZZ} - \Sigma_{ZX} \Lambda^{-1}\Sigma_{XZ}
\end{array}
\end{equation}
where $\Lambda \triangleq \begin{bmatrix}
\Sigma_{X_1X_1}+\Sigma_n & \Sigma_{X_1X_B} \notag\\ 
\Sigma_{X_BX_1} &  \Sigma_{X_BX_B} + \tilde{\Sigma}
\end{bmatrix}$, $\tilde{y}_X \triangleq [y_{X_1}; \tilde{\mu}]$, and $\Sigma_n$ is a $|X_1| \times |X_1|$ diagonal matrix with diagonal components $\sigma^2_{n_1}$.
Consequently, the approximated predictive belief of $y_i(x)$ for any input tuple $\langle x, i\rangle \in D^+$ can be computed using
$p(y_i(x)|y_X) = \int p(y_i(x)|f_i(x))\ p(f_i(x)|y_X)\ \text{d}f_i(x)$.
Due to \eqref{l} and \eqref{var},
\begin{equation}\label{py}
\hspace{-1.7mm}
\begin{array}{rcl}
p(y_1(x)|y_X)  & \hspace{-2.4mm}\approx& \hspace{-2.4mm} \mathcal{N}(y_1(x)|\mu_{\{\langle x, 1\rangle\}|X}, \sigma^2_{\langle x, 1\rangle|X}\hspace{-1mm} + \hspace{-0.5mm}\sigma^2_{n_1}) \vspace{0.5mm}\\
p(y_i(x) = 1|y_X)  & \hspace{-2.4mm}\approx& \hspace{-2.4mm}  \Phi_\text{cdf} \left(\mu_{\{\langle x, i\rangle\}|X}/\sqrt{1+\sigma^2_{\langle x, i\rangle|X}} \ \right)
\end{array}
\end{equation}
for $i = 2, \ldots, M$
where $\sigma^2_{\langle x, i\rangle|X} \triangleq \Sigma_{\{\langle x, i\rangle\}\{\langle x, i\rangle\}|X}$ for $i = 1, \ldots, M$.
\section{BO WITH BINARY AUXILIARY INFORMATION}\label{acqfn}
To achieve the objective described in Section~\ref{setup}, our BO algorithm repeatedly selects the next input tuple $\langle x, i \rangle$ for evaluating the $i$-th function $f_i$ at $x$
that maximizes a choice of acquisition function $\alpha (y_{X},\langle x, i \rangle)$ \emph{per unit cost} given the past observations $(X, y_{X})$:
$$
\begin{array}{c}
\langle x, i \rangle^+ \triangleq\mathop{\arg\max}_{\langle x, i \rangle \in D^{+} \setminus X} \alpha (y_{X},\langle x, i \rangle) / \lambda_i(x)
\end{array}
$$
and updates $X\leftarrow X\cup\{\langle x, i \rangle^+ \}$ until the budget is expended.
Since the costs of evaluating the target vs. auxiliary functions differ, we use the above \emph{cost-sensitive} acquisition function such that the cheaper auxiliary function evaluations can be exploited. We will focus on designing the acquisition function $\alpha$ first and the estimation of $\lambda_i(x)$ in real-world applications will be discussed later in Section~\ref{experiment}.

Intuitively, 
$\alpha$ should be designed to enable its BO algorithm to jointly and naturally optimize the non-trivial trade-off between 
exploitation vs. exploration over the target and auxiliary functions for finding or improving the belief of the global target maximizer $x_*$ by utilizing information from the mixed-type CMOGP predictive belief of these functions~\eqref{var}. 
To do this, one may be tempted to directly use the conventional EI \citep{mockus1978} and $\text{EI}_\pi$ \citep{tesch13} acquisition functions 
for selecting inputs to evaluate the target and auxiliary functions, respectively.
$\text{EI}_\pi$ is a variation of EI and, to the best of our knowledge, the only acquisition function designed for optimizing an unknown function with a binary output type.
%
However, this does not satisfy our objective since $\text{EI}_\pi$ aims to find the global maximizer of the auxiliary function which can differ from the global target maximizer if the target and auxiliary functions are not perfectly correlated.
%
To resolve this issue, we propose to exploit information-based acquisition functions and generalize them to our mixed-type BO problem such that input tuples for evaluating the target and auxiliary functions are selected to directly maximize \emph{only} the unknown target objective function, as detailed later.\vspace{-1mm}
\subsection{INFORMATION-BASED ACQUISITION FUNCTIONS FOR MIXED-TYPE BO}\label{Info}\vspace{-0.6mm}
Information-based acquisition functions like ES
and PES 
have been designed to enable their BO algorithms to improve the belief of the global target maximizer.
In mixed-type BO, we can similarly define a belief of the maximizer $x_{*_i}$ of each $i$-th function $f_i$ as
$
p(x_{*_i}|y_X) \triangleq p(f_i(x_{*_i}) = \mathop{\max}_{x \in D} f_i(x) | y_X)
$
for $i=1, ..., M$.
To achieve the objective of maximizing \textit{only} the target function in mixed-type BO, ES can be used to measure the information gain of \textit{only} the global target maximizer $x_{*}$ (i.e., $x_{*_1}$) from selecting the next input tuple $\langle x, i \rangle$ for evaluating the $i$-th (possibly binary auxiliary) function $f_i$ 
at $x$ given the past observations $(X,y_X)$:
\begin{equation} \label{ES}
\alpha(y_X,\hspace{-0.5mm} \langle x, i \rangle ) \hspace{-0.5mm}\triangleq \hspace{-0.5mm} H(x_{*}|y_X)- \mathbb{E}_{p(y_i(x) | y_X)}[H(x_{*}|y_{X \cup \{\langle x, i \rangle\}})].
\end{equation}
Similar to the \emph{multi-task ES} algorithm~\citep{Swersky2013} which is designed for BO with \emph{continuous} auxiliary information, we can use Monte Carlo sampling to approximate~\eqref{ES} by utilizing information from the mixed-type CMOGP predictive belief (i.e.,~\eqref{var} and~\eqref{py}) of the target and auxiliary functions.
To make the Monte Carlo approximation tractable and efficient, we need to discretize the input domain and assume that the search space for evaluating \eqref{ES} is pruned to a small set of input candidates which, following the work of~\citet{Swersky2013}, can be selected by applying EI to \emph{only} the target function.
Such a form of approximation, however, faces two critical limitations: (a) Computing~\eqref{ES} incurs cubic time in the size of the discretized input domain and is thus expensive to evaluate with a large input domain (or risks being approximated poorly), and
(b) the pruning of the search space artificially constrains the exploration of auxiliary functions and requires a parameter in EI (i.e., to control the exploration-exploitation trade-off) to be manually tuned to fit different real-world applications.


To circumvent the above-mentioned issues, we can exploit the symmetric property of conditional mutual information and rewrite~\eqref{ES} as
\begin{equation}\label{PES}
\alpha(y_X,\hspace{-0.5mm} \langle x, i \rangle ) \hspace{-0.5mm}
=\hspace{-0.5mm} H(y_i(x) |y_X) - \mathbb{E}_{p(x_{*}| y_X)}[H(y_i(x) | y_X, x_{*})]
\end{equation}
which we call \emph{mixed-type PES} (MT-PES).
Intuitively, the selection of an input tuple $\langle x, i \rangle$ to maximize~\eqref{PES} has to trade off between exploration of every target and auxiliary function (hence inducing a large Gaussian predictive entropy $H(y_i(x)|y_X)$) vs. exploitation of the current belief $p(x_{*}| y_X)$ of the global target maximizer $x_{*}$ to choose a nearby input $x$ of  function $f_i$ (i.e., convolutional structures and maximizers of the target and auxiliary functions are similar or close (Section~\ref{CMOGP})) to be evaluated (hence inducing a small expected predictive entropy $\mathbb{E}_{p(x_{*}| y_X)}[H(y_i(x)| y_X, x_{*})]$) to yield a highly informative observation that in turn improves the belief of $x_{*}$.
Note that the entropy of continuous random variables (i.e., differential entropy) and discrete/binary random variables (i.e., Shannon entropy) are not comparable\footnote{For example, the Shannon entropy is always non-negative while the differential entropy can be negative. A detailed discussion of their difference and connection is available in Chapter 8 of \cite{cover2006}.}.
So, the differential entropy terms in \eqref{PES} for $i=1$ are not comparable to the Shannon entropy terms in \eqref{PES} for $i = 2, \ldots, M$. 
Fortunately, the difference of the two entropy terms in \eqref{PES} 
is exactly the information gain of the global target maximizer $x_*$ in~\eqref{ES} which is comparable between $i = 1$ vs.~$i=2, \ldots, M$ regardless of whether the output $y_i(x)$ is continuous or binary. 
Next, we will describe how to evaluate \eqref{PES} efficiently.
\section{APPROXIMATION OF MIXED-TYPE PREDICTIVE ENTROPY SEARCH} \label{MFPES}
Due to~\eqref{py},
the first Gaussian predictive/posterior entropy term in~\eqref{PES} can be computed analytically:
\begin{equation}\label{puke}
\hspace{-1.7mm}
\begin{array}{rcl}
H(y_1(x)|y_X) & \hspace{-2.4mm} \triangleq & \hspace{-2.4mm} 0.5 \log( 2\pi e (\sigma^2_{\langle x, 1 \rangle | X} + \sigma^2_{n_1})) \vspace{1mm}\\
H(y_i(x)|y_X) &\hspace{-2.4mm} \triangleq& \hspace{-2.4mm}\displaystyle -
\hspace{-1mm}
\sum_{y_i(x) \in \{1, -1\}} 
\hspace{-0.5mm} 
p(y_i(x)|y_X)\log p(y_i(x)|y_X) 
\end{array}
\end{equation}
for $i = 2, \ldots, M$.
Unfortunately, the second term in~\eqref{PES} cannot be evaluated in closed form.
Although this second term appears to resemble that in PES~\citep{hernandez2014}, their approximation method, however, cannot be applied straightforwardly here since it cannot account for either the binary auxiliary information or the complex cross-correlation structure between the target and auxiliary functions.
To achieve this, we will first propose a novel mixed-type random features approximation of the CMOGP model whose cross-correlation structure between the target and auxiliary functions can be exploited for sampling the global target maximizer $x_{*}$ more accurately using the past observations $(X,y_X)$ from evaluating these functions (especially when the target function is sparsely evaluated due to its higher cost), which is in turn used to approximate the expectation in~\eqref{PES}.
Then, we will formalize some practical constraints relating the global target maximizer to the binary auxiliary functions, which are used to approximate the second entropy term within the expectation in~\eqref{PES}.
%
%
\subsection{MIXED-TYPE RANDOM FEATURES}
\label{MORF}
To approximate the expectation in \eqref{PES} efficiently by averaging over samples of the target maximizer from $p(x_{*}|y_X)$ in a continuous input domain, we will derive an analytic sample of the unknown function
$f_i$ given the past observations $(X,y_X)$, which is differentiable and can be optimized by any existing gradient-based optimization method to search for its maximizer.
Unlike the work of~\citet{hernandez2014} that achieves this in PES using 
the \emph{single-output random features} (SRF)
method for handling a single continuous output type~\citep{Miguel10,rahimi2007}, we have to additionally consider how the binary auxiliary functions and their complex cross-correlation structure with the target function can be exploited for sampling the target maximizer $x_{*}$ more accurately.
To address this, we will now present a novel \emph{mixed-type random features} (MT-RF) approximation of the CMOGP model by first deriving an analytic form of the latent function $L$ with SRF and then an analytic approximation of $f_i$
using the convolutional structure of the CMOGP model. The results of EP \eqref{EP} can be reused here to approximate the non-Gaussian likelihood $p(y_i(x)|f_i(x))$ for $i = 2, \ldots, M$.




Using SRF~\citep{rahimi2007}, the latent function $L$ modeled using GP can be approximated by a linear model $L(x) \approx \phi (x)^\top \theta$ where $\phi(x)$ is a random vector of an $m$-dimensional feature mapping of the input $x$ for $L(x)$ and $\theta \sim \mathcal{N}(\underline{0}, I)$ is an $m$-dimensional vector of weights. Then, interestingly, by exploiting the convolutional structure of the CMOGP model in \eqref{func}, $f_i(x)$ can also be approximated analytically by a linear model: $f_i(x) \approx m_i + \phi_i(x)^\top\theta$ where the random vector $\phi_i(x) \triangleq \sigma_{s_i}\ \text{diag}(\exp({-0.5W^\top P_i^{-1} W}))\ \phi(x)$ can be interpreted as input features of $f_i(x)$,
$W$ is a $d \times m$ random matrix which is used to map $x \rightarrow \phi(x)$ in SRF,  and function $\text{diag}(A)$ returns a diagonal matrix with the same diagonal components as $A$. The exact definition of $\phi(x)$ and the derivation of $\phi_i(x)$ are in Appendix~\ref{a_conv}.

Then, a sample of $f_i$ can be constructed using $f^{(s)}_i(x) \triangleq m_i + \phi_i^{(s)}(x)^\top \theta^{(s)}$ where $\phi_i^{(s)}(x)$ and $\theta^{(s)}$ are vectors of features and weights sampled, respectively, from the random vector $\phi_i(x)$ and the posterior distribution of weights $\theta$ given the past observations $(X, y_X)$,
the latter of which is approximated to be Gaussian by exploiting the conditional  independence property of MT-RF and the results of EP~\eqref{EP} from the mixed-type CMOGP model:
$$
p(\theta|y_X) = \mathcal{N}(\theta|A^{-1}\Phi(\Lambda - \Sigma_{XX})^{-1}(\tilde{y}_X - \mu_X), A^{-1})
$$
where $A \triangleq \Phi (\Lambda - \Sigma_{XX})^{-1} \Phi^\top \hspace{-1mm} + I$ and $\Phi \triangleq (\phi_j(x))_{\left\langle x, j \right\rangle \in X}$, as detailed in Appendix \ref{a_conv2}.

Consequently, the expectation in~\eqref{PES} can be approximated by averaging over $S$ samples of the target maximizer $x^{(s)}_{*}$ of $f^{(s)}_1$ to yield an approximation of MT-PES:
\begin{equation}\label{approx0}
\alpha(y_X,\langle x, i \rangle ) \approx H(y_i(x)|y_X) - \frac{1}{S}\sum_{s=1}^S H(y_i(x) | y_X, x^{(s)}_{*})
\end{equation}
where
$x^{(s)}_{*} \triangleq x^{(s)}_{*_1}$ and $x^{(s)}_{*_i} \triangleq \mathop{\arg\max}_{x \in D} f_i^{(s)}(x)\ $ for $i=1,\ldots,M$.
%
Drawing a sample of $x^{(s)}_{*}$ incurs $\mathcal{O}(m^3+ m^2|X|)$ time if $m \leq |X|$ and $\mathcal{O}(|X|^3+ |X|^2m)$ time if $m > |X|$, which 
is more efficient than using Thompson sampling to sample $f_i$ over a discretized input domain that incurs cubic time in its size since a sufficiently fine discretization of the entire input domain is typically larger in size than the no. $|X|$ of observations.\vspace{-1.3mm} 
\subsection{APPROXIMATING THE PREDICTIVE ENTROPY CONDITIONED ON THE TARGET MAXIMIZER}
\label{approxPES}\vspace{-1mm} 
We will now discuss how the second entropy term in~\eqref{approx0} is approximated. Firstly, the posterior distribution of $y_i(x)$ given the past observations 
and target maximizer
is computed by
\begin{equation}\label{p}
\hspace{-1.7mm}
\begin{array}{c}
 p(y_i(x)| y_X, x_{*}) 
\displaystyle \hspace{-0.5mm} = \hspace{-1mm}\int p(y_i(x)|f_i(x))\ p(f_i(x)|y_X, x_{*})\ \text{d}f_i(x)
\end{array}
\end{equation}
where $p(y_i(x)|f_i(x))$ is defined in \eqref{l} and $p(f_i(x)|y_X, x_{*})$ will be approximated by EP, as detailed later.
As shown in Section~\ref{CMOGP}, the Gaussian predictive belief $p(f_i(x)|y_X)$~\eqref{var} can be computed analytically. Then, $p(f_i(x)|y_X, x_{*})$ can be considered as a constrained version of $p(f_i(x)|y_X)$ by further conditioning on the target maximizer $x_{*}$. 
It is intuitive that the posterior distribution of $f_i(x)$ is constrained by\vspace{-0mm}
$f_i(x) \leq f_i(x_{*_i}), \forall \langle x, i \rangle \in D^+$.
However, since only the target maximizer $x_{*}$ is of interest, how should the value of $f_i(x)$ be constrained by $x_{*}$ instead of $x_{*_i}$ if $i = 2, \ldots, M$? 
To resolve this, we introduce a slack variable $c_i$ to formalize the relationship between maximizers of the target and auxiliary functions:\vspace{-0.5mm}
\begin{equation} \label{cons2}
f_i(x) \leq f_i(x_{*})+c_i\quad  \forall x \in D, i \neq 1\vspace{-0mm}
\end{equation}
where $c_i \triangleq \mathbb{E}_{p(x_{*_i}|y_X)}[f_i(x_{*_i})] - \mathbb{E}_{p(x_{*}|y_X)}[f_i(x_{*})]$ measures the gap between the expected maximum of $f_i$ and the expected output of $f_i$ evaluated at $x_{*}$ and can be approximated efficiently using our MT-RF method even though $f_i$ is unknown, as detailed later.
Consequently, the following simplified constraints instead of \eqref{cons2} will be used to approximate $p(f_i(x)|y_X, x_{*})$:\vspace{-1mm}
\begin{enumerate}
	\item [$C1$.]  $f_i(x) \hspace{-0.5mm}\leq \hspace{-0.5mm} f_i(x_{*}) \hspace{-0.3mm} + \hspace{-0.3mm} \delta_i c_i$ for a given $\langle x, i \rangle \hspace{-0.5mm} \in \hspace{-0.4mm} D^+\hspace{-0.7mm}$ where $\delta_i$ equals to $0$ if $i = 1$, and $1$ otherwise.\vspace{-0mm}
	\item [$C2$.] $f_1(x_{*}) \geq y_{\text{max}} + \epsilon_1$ where 
	$\epsilon_1 \sim \mathcal{N}(0, \sigma^2_{n_1})$ and 
	$y_{\text{max}} \triangleq \max_{\langle x,1\rangle\in X_1} y_1(x)$ is the largest among the noisy outputs observed by evaluating the target function $f_1$ at $X_1$. \vspace{-0mm}
	\item [$C3$.] $\Phi_\text{cdf}(f_j(x_{*}) + c_j) \geq 0.5$ for $j = 2, \ldots, M$.\footnote{Like the work of~\citet{Swersky2013} (Section~$2.2$), we assume the cross-correlation between the target and auxiliary functions to be positive. An auxiliary function that is negatively correlated with the target function can be easily transformed to be positively correlated by negating all its outputs.}\vspace{-1mm}
\end{enumerate}
The first constraint $C1$ keeps the influence of $x_{*}$ to the next input tuple $\langle x,i\rangle$ to be selected by MT-PES. Instead of constraining all unknown functions over the entire input domain, $C2$ and $C3$ relax~\eqref{cons2} to be valid only for the outputs observed from evaluating these functions. When the target and auxiliary functions are highly correlated (i.e., small $c_j$), $C3$ means that a positive label can be observed with high probability by evaluating an auxiliary function at the target maximizer $x_*$.
Using these constraints, 
$
p(f_i(x)|y_X, x_{*}) \approx p(f_i(x)|y_X, C1, C2, C3)
$
which can be approximated analytically using EP. To achieve this, we will first derive a tractable approximation of the posterior distribution $p(f_i(x_{*})|y_X, C2, C3)$ which does not depend on the next selected input $x$. Note that such terms can be computed once and reused in the approximation of $p(f_i(x)|y_X, x_{*})$ in \eqref{p} which depends on $x$, as detailed later.\vspace{0mm}
%
%

\textbf{Approximating terms independent of $x$.}
Let $f^*_j \triangleq f_j(x_{*})$ and $f^* \triangleq (f_j^*)^{\top}_{j=1,\ldots,M}$. We can use the cdf of a standard Gaussian distribution and an indicator function to represent the probability of $C2$ and $C3$, respectively. Then, the posterior distribution $p(f^*|y_X)$ can be constrained with $C2$ and $C3$ by
\begin{equation} \label{approx1}
\hspace{-1.7mm}
\begin{array}{l}
p(f^*|y_X, C2, C3) \vspace{0.5mm}\\
\displaystyle \propto p(f^*|y_X)\ \Phi_\text{cdf}\hspace{-0.7mm}\left( \hspace{-0.5mm}\frac{f^*_1 - y_{\text{max}}}{\sigma_{n_1}} \hspace{-0.5mm}\right) \prod_{j=2}^M \mathbb{I}(f^*_j + c_j \geq 0)\ . \hspace{-4.4mm}
\end{array}
\end{equation}
Interestingly, by sampling the target and auxiliary maximizers $x_{*}$ and $x_{*_j}$ using our MT-RF method proposed in Section \ref{MORF}, the value of $c_j$ in \eqref{approx1} can be approximated by Monte Carlo sampling\footnote{When $j=1$, $c_j$ is equal to $0$ since $x_{*_j} = x_{*}$.}:\vspace{-0mm}
$$
\begin{array}{rcl}
c_j &\hspace{-2.4mm} =&\hspace{-2.4mm} \mathbb{E}_{p(x_{*_j}|y_X)}[f_j(x_{*_j})] - \mathbb{E}_{p(x_{*}|y_X)}[f_j(x_{*})] \vspace{0.5mm}\\
&\hspace{-2.4mm}  \approx&\hspace{-2.4mm} {S}^{-1}\sum^S_{s=1}\left( f^{(s)}_j(x^{(s)}_{*_j}) - f^{(s)}_j(x^{(s)}_{*}) \right).
\end{array}
$$
With the multiplicative form of~\eqref{approx1} ,  $p(f^*|y_X, C2, C3)$ can be approximated to be a multivariate Gaussian  
$\mathcal{N}(f^*|\mu, \Sigma)$ using EP by approximating each non-Gaussian factor (i.e., $\Phi_\text{cdf}$ and $\mathbb{I}$) in \eqref{approx1} to be a Gaussian, as detailed in Appendix~\ref{A_EP1}. Consequently, the posterior distribution $p(f_i^*|y_X, C2, C3)$ can be approximated by a Gaussian $\mathcal{N}(f_i^*|\mu_i, \tau_i)$
where $\mu_i$ is the $i$-th component of $\mu$ and $\tau_i$ is the $i$-th diagonal component of $\Sigma$.
%
%

\textbf{Approximating terms that depend on $x$.}
In $C2$ and $C3$, $f_i^*$ is the only term that is related to $C1$. It follows that $f_i(x)$ is conditionally independent of $C2$ and $C3$ given $f_i^*$. Let $f^+ \triangleq [f_i(x_{*}); f_i(x)]$. So, 
$
p(f^+| y_X, C2, C3) = p(f_i(x)|y_X,f_i^*)\ p(f_i^*|y_X, C2, C3) = \mathcal{N}(f^+|\mu^+, \Sigma^+)
$
where $\mu^+$ and $\Sigma^+$ can be computed analytically using $\mu_i$, $\tau_i$,  and~\eqref{var}, as detailed in Appendix~\ref{A_G1}.

To involve $C1$, an indicator function $\mathbb{I}(f_i(x) \leq f_i(x_{*}) + \delta_i c_i)$ is used to represent the probability that $C1$ holds. Then, 
$p(f_i(x)|y_X, x_{*}) \approx \int p(f^+|y_X, C1, C2, C2) \ \text{d}f^*_i$ where
\begin{equation} 
\hspace{-1.7mm}
\begin{array}{l}
p(f^+|y_X, C1, C2, C3) \vspace{0.5mm}\\
\displaystyle \approx {Z'}^{-1} p(f^+|y_X, C2)\ \mathbb{I}(f_i(x) \leq f_i(x_{*}) + \delta_i c_i)\ . \hspace{-1.5mm}
\end{array}
\label{approx2}
\end{equation}
Since the posterior of $f_i(x_{*})$ has been updated according to $C2$ and $C3$~\eqref{approx1}, $c_i$ in~\eqref{approx2} is updated likewise:
$$
\begin{array}{c}
c_i \approx {S}^{-1}\sum^S_{s=1}\left( f^{(s)}_i(x^{(s)}_{*_i}) - \mu^{(s)}_i \right)
\end{array}
$$
where $\mu^{(s)}_i$ is computed in~\eqref{approx1} using a sampled $x^{(s)}_{*}$.
Similar to that in \citep{hernandez2014}, a one-step EP can be used to approximate~\eqref{approx2} as a multivariate Gaussian with the following posterior mean vector and covariance matrix:
\begin{equation}
\hspace{-1.7mm}
\begin{array}{rcl}
\mu_{f^+} & \hspace{-2.4mm}\triangleq & \hspace{-2.4mm}  \mu^+ - (\gamma / \sqrt{v}) \Sigma^+ a \vspace{1mm}\\
\Sigma_{f^+} & \hspace{-2.4mm}\triangleq & \hspace{-2.4mm} \Sigma^+ - v^{-1} \gamma (\gamma - (\eta-\delta_ic_i)/{\sqrt{v}})\ \Sigma^+ aa^\top \Sigma^+ \hspace{-7mm}\vspace{-5mm}
\end{array}
\label{EPr}\vspace{4mm}
\end{equation}
where $\gamma \triangleq \phi(({\delta_i c_i - \eta})/{\sqrt{v}})/ \Phi_\text{cdf}(({\delta_i c_i - \eta})/{\sqrt{v}})$, $\eta \triangleq a^\top \mu^+$, $v \triangleq a^\top \Sigma^+ a$ and $a = [-1; 1]$. The derivation of~\eqref{EPr} is in Appendix~\ref{A_EPr}. 
So, the posterior mean and variance of $p(f_i(x)|y_X, x_{*})$ can be approximated, respectively, using the $2$-th component of $\mu_{f^+}$ and $(2,2)$-th component of $\Sigma_{f^+}$ denoted by $\mu_{f_i}$ and $v_{f_i}$.
As a result, the posterior entropy $H(y_i(x) | y_X, x_{*}^{(s)})$ in \eqref{approx0} can be approximated using \eqref{puke} by replacing $\mu_{\{\langle x, i\rangle\}|X}$ and $\sigma^2_{\langle x, i\rangle|X}$  in \eqref{puke} with, respectively, $\mu^{(s)}_{f_i}$ and $v^{(s)}_{f_i}$ where $\mu^{(s)}_{f_i}$ and $v^{(s)}_{f_i}$ are computed in \eqref{EPr} using a sampled $x_*^{(s)}$.\vspace{-1mm}
%
%
\section{EXPERIMENTS AND DISCUSSION}\vspace{-0.3mm}
\label{experiment}
This section empirically evaluates the performance of our MT-PES algorithm against that of (a) the state-of-the-art PES \citep{hernandez2014} without utilizing the binary auxiliary information and
(b) MT-ES performing Monte Carlo approximation of~\eqref{ES}.
In all experiments, we use $m\triangleq 200$ random features and $S\triangleq 50$ samples of the target maximizer in MT-PES. The input candidates with top $30$ EI values are selected for evaluating MT-ES. The \emph{mixed-type MOGP} (MT-MOGP) hyperparameters are learned via maximum likelihood estimation. The performance of the tested algorithms are evaluated using \emph{immediate regret} (IR) $|f_1(x_{*}) - f_1(\tilde{x}_{*})|$ where 
$\tilde{x}_{*} \triangleq \mathop{\arg\max}_{x \in D} \mu_{\{\left\langle x, 1 \right\rangle\}|X}$ is their recommended target maximizer. In each experiment, one observation of the target function is randomly selected as the initialization.\vspace{-1mm}
%
%
%
\subsection{SYNTHETIC EXPERIMENTS} \vspace{-1mm}
%
The performance of the tested algorithms are firstly evaluated using synthetic and benchmark functions.
\begin{figure*}
	\centering
\begin{tabular}{ccccc}
		\hspace{-2.5mm}\includegraphics[scale=0.20]{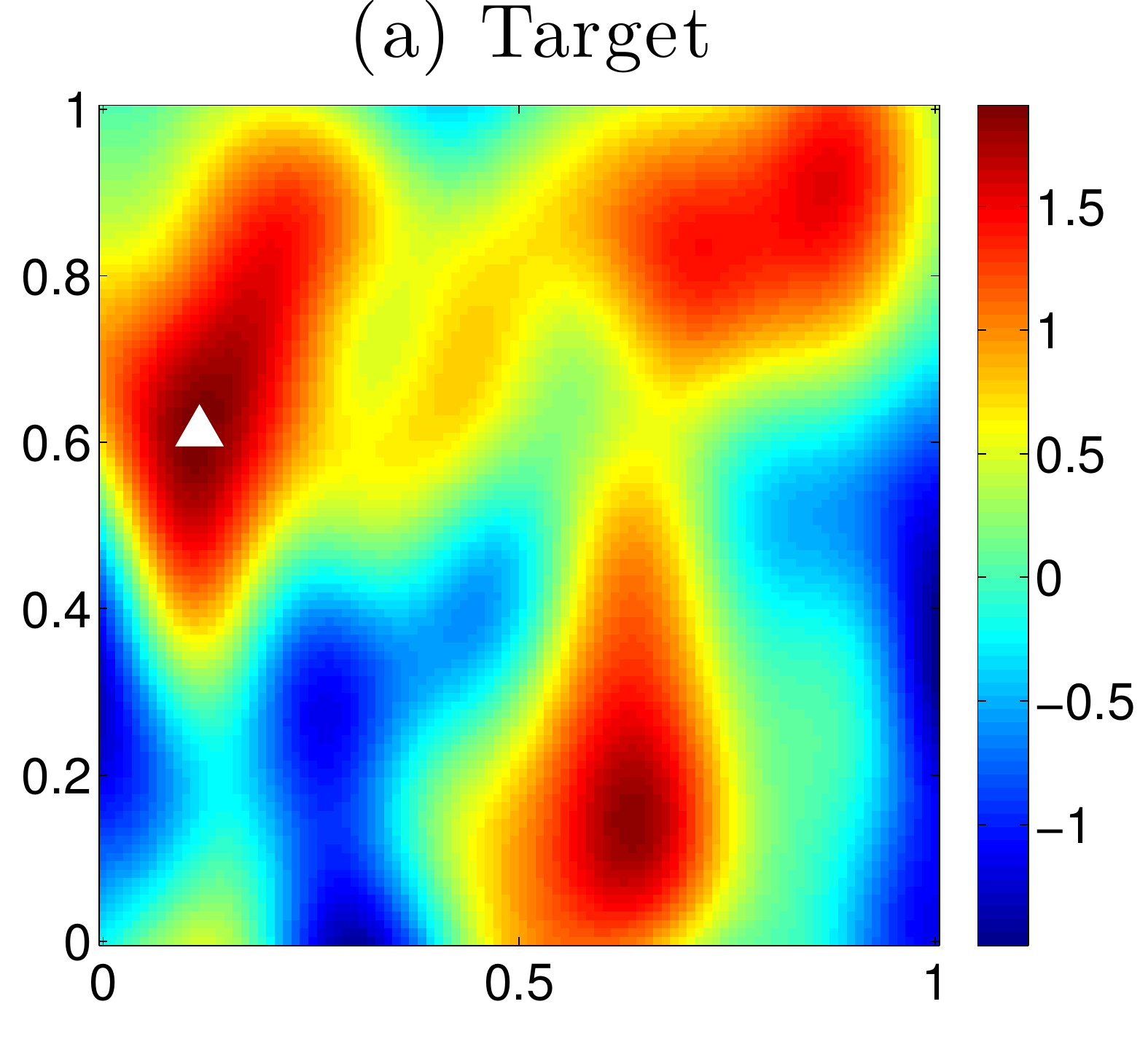} &
		\hspace{-3mm}\includegraphics[scale=0.20]{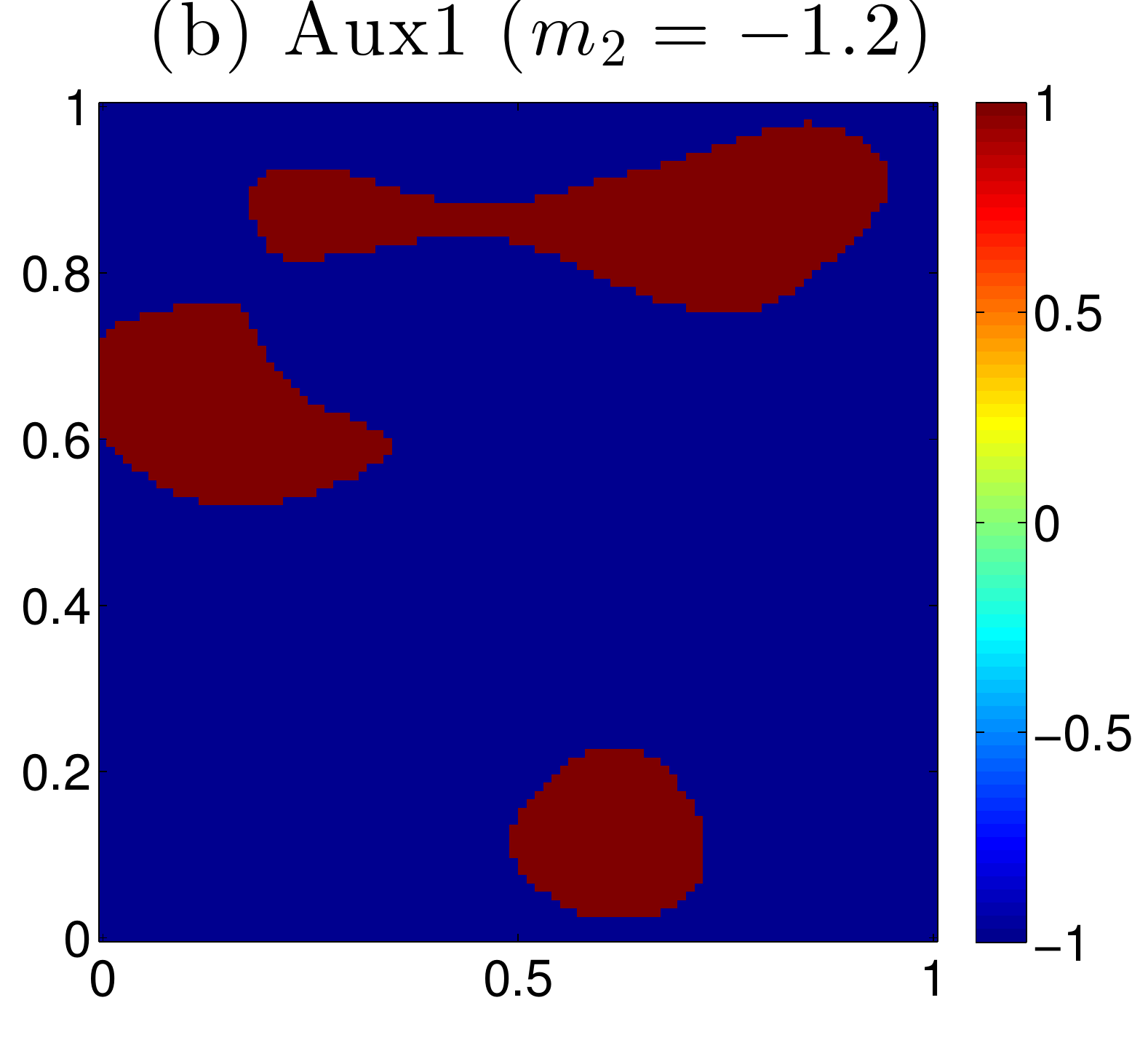} &
		\hspace{-3mm}\includegraphics[scale=0.20]{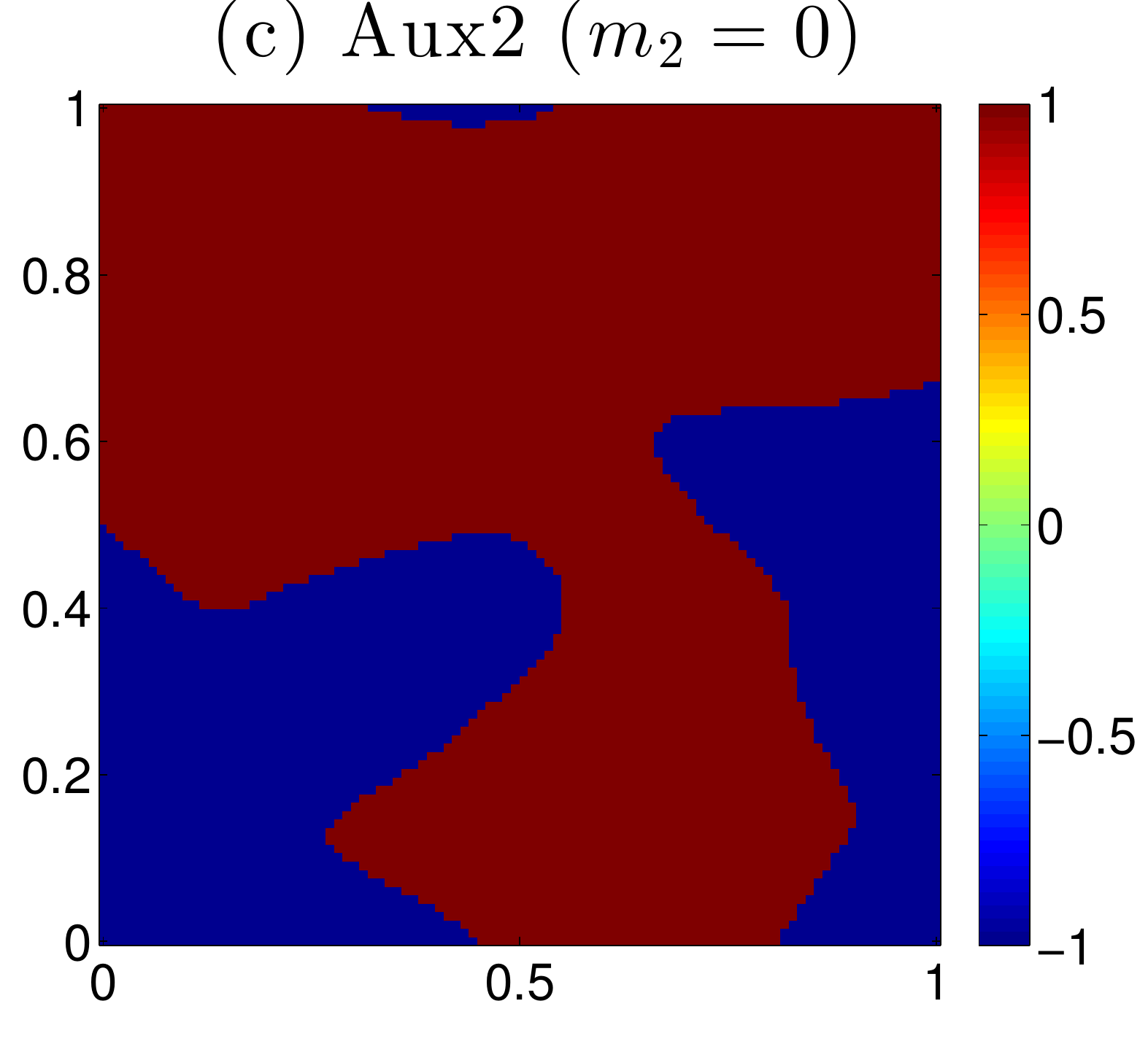} &
		\hspace{-3mm}\includegraphics[scale=0.20]{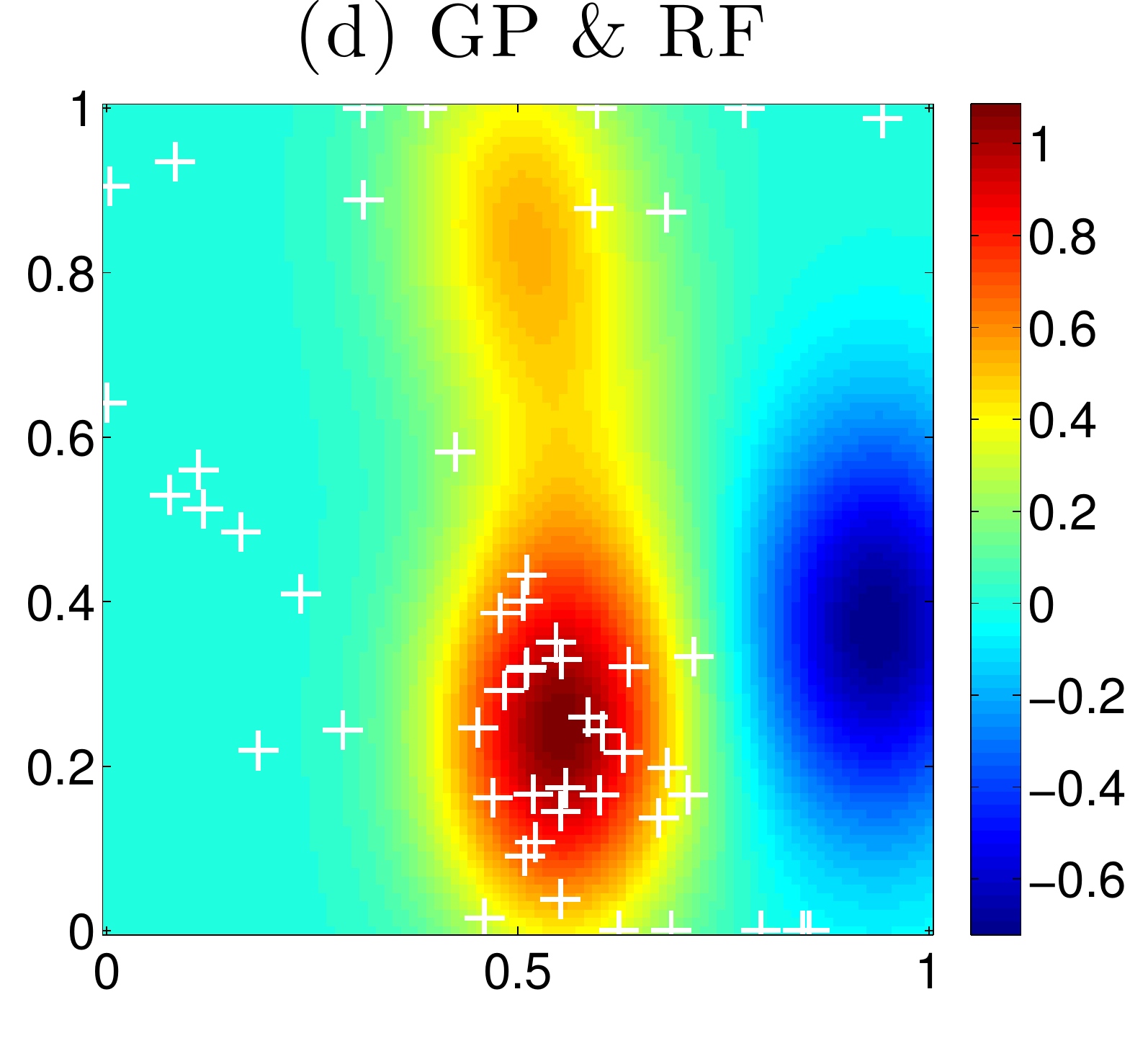} & 
		\hspace{-3mm}\includegraphics[scale=0.20]{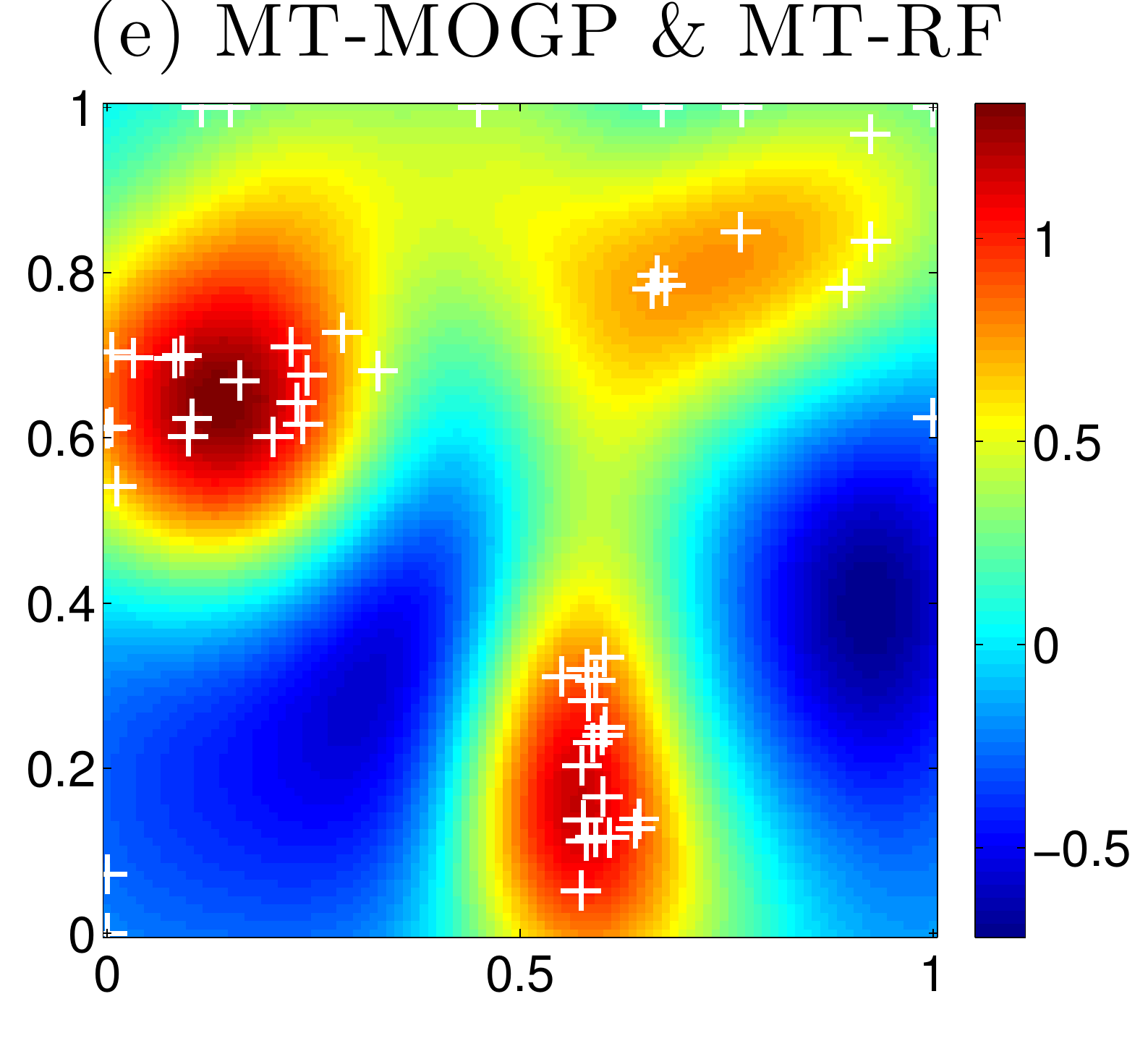}\vspace{-4mm}\\
	\end{tabular}
	\caption{(a-c) Example of the synthetic functions where `$\triangle$' is the global target maximizer, (d) target function predicted by conventional GP model and the target maximizers (`$+$' ) sampled by RF with $5$ observations from evaluating the target function, and (e) target function predicted by MT-MOGP model and the target maximizers (`$+$' ) sampled by MT-RF with $5$ and $50$ observations from evaluating the target and aux1 functions, respectively.}\vspace{-1mm}
	\label{fig:syn}
\end{figure*}
\begin{figure*}
	\centering
	\begin{tabular}{ccc}
		\hspace{-0mm}\includegraphics[scale=0.22]{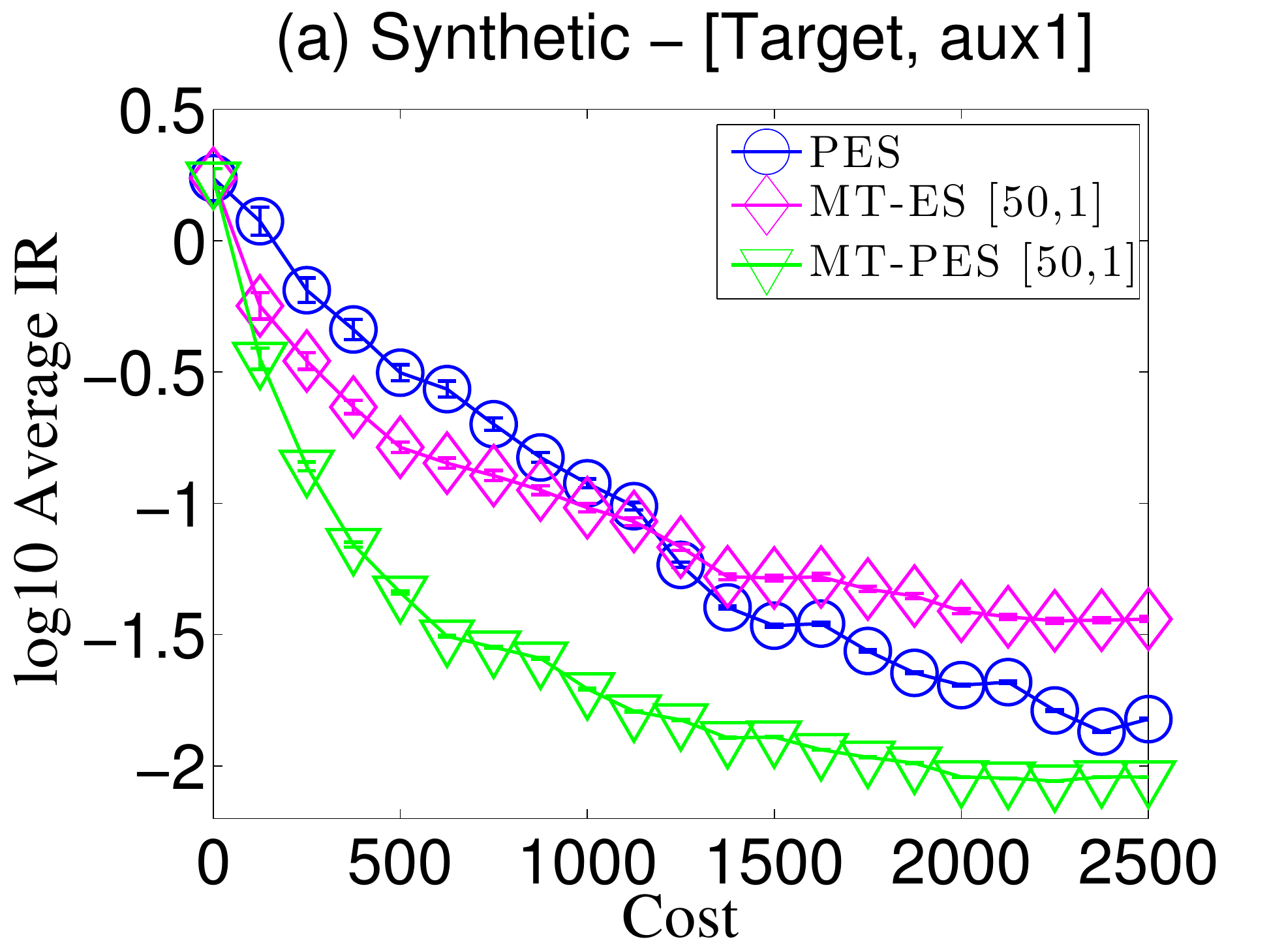} &
		\hspace{-0mm}\includegraphics[scale=0.22]{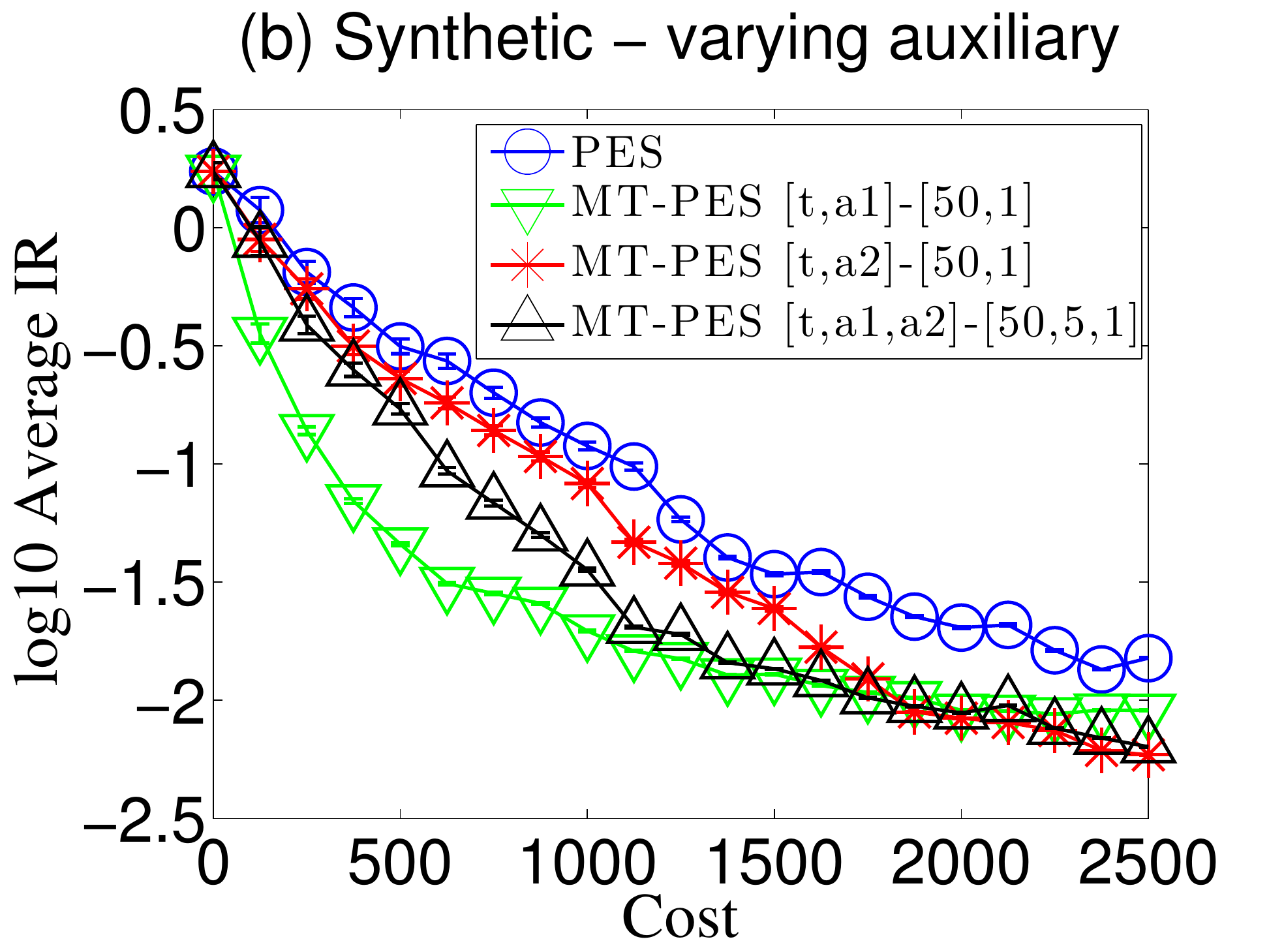} &
 		\hspace{-0mm}\includegraphics[scale=0.22]{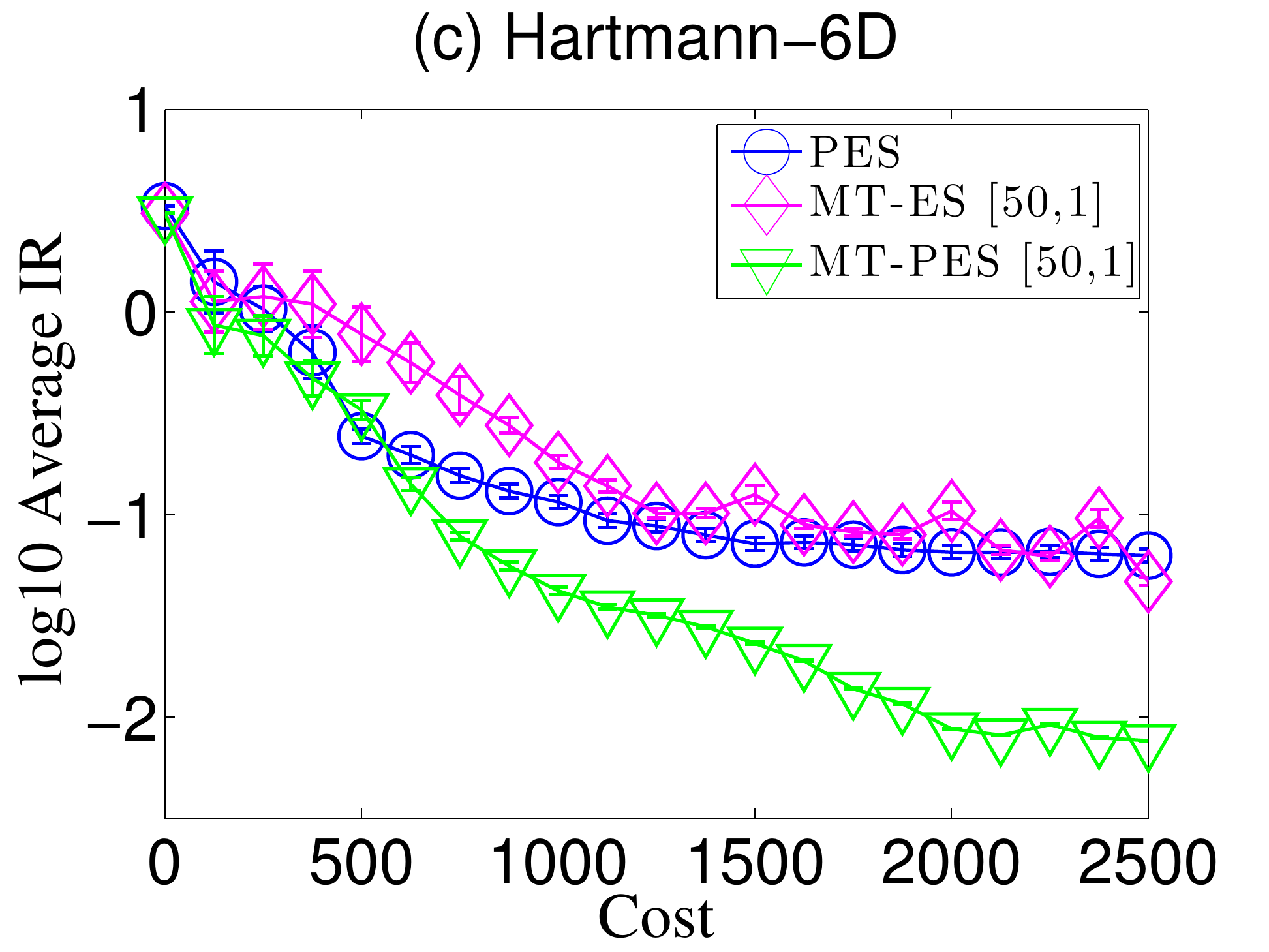}
		 		\vspace{-4mm} \\
	\end{tabular}
	\caption{Graphs of $\log_{10}(\text{averaged IR})$ vs. cost incurred by tested algorithms for (a-b) synthetic functions and (c) Hartmann-6D function. The type and cost of functions used in each experiment are shown in the title and legend of each graph where `t' denotes target function and `a1' and `a2' denote aux1 and aux2 functions, respectively. The error bars are computed in the form of standard error.}\vspace{-2mm}
	\label{fig:syn1}
\end{figure*}

{\bf Synthetic functions.} The synthetic functions are generated using $M\triangleq 2$ and $D \triangleq [0, 1]^2$. To do this, the CMOGP hyperparameters with one latent function are firstly fixed as the values shown in Appendix \ref{a_syn},
which are also used in the tested algorithms as optimal hyperparameters. 
Then, a set $X$ of $450$ input tuples are uniformly sampled from $D^+$ and their corresponding outputs are sampled from the CMOGP prior. The target function is set to be the predictive mean $\mu_{\{\langle x, 1 \rangle\}|X}$ of the CMOGP model. The outputs of the auxiliary function are set to be $1$ if $\mu_{\{\langle x, 2 \rangle\}|X} \geq 0$, and $-1$ otherwise.
An example of the synthetic functions can be found in Figs.~\ref{fig:syn}a to~\ref{fig:syn}c.
As can be seen in Figs.~\ref{fig:syn}b and~\ref{fig:syn}c, we can generate multiple auxiliary functions with different proportions of positive outputs from a target function (Fig.~\ref{fig:syn}a) by varying the bias $m_2$. All these auxiliary functions correlate well with the target function but 
delineate the input regions containing the target maximizer differently 
and thus result in different MT-PES performance, as will be shown later.

{\bf Empirical analysis of MT-MOGP and MT-RF.}
Firstly, we verify that the MT-MOGP model and MT-RF can outperform the conventional GP model and single-output RF by exploiting cross-correlation structure between the target and auxiliary function aux1 (i.e., Figs.~\ref{fig:syn}a and~\ref{fig:syn}b).
Figs.~\ref{fig:syn}d and~\ref{fig:syn}e show the predictive mean and the sampled maximizers of the target function using randomly sampled observations.
By comparing Figs.~\ref{fig:syn}d and~\ref{fig:syn}e with Fig.~\ref{fig:syn}a, it can be observed that the MT-MOGP model and MT-RF can predict the target function and sample the target maximizer more accurately than the conventional GP model and single-output RF using an additional $50$ observations from evaluating aux1.

{\bf Empirical analysis of mixed-type BO.}
Next, the performance of the tested BO algorithms are evaluated using ten groups (i.e., one target function, two auxiliary functions aux1 and aux2 with different $m_2$) of synthetic functions generated using the above procedure. We adjust $m_2$ such that around $20\%$ of auxiliary outputs are positive for each aux1 and set $m_2 = 0$ for each aux2. An averaged IR is obtained by optimizing the target function in each of them with $10$ different initializations for each tested algorithm.

Fig.~\ref{fig:syn1} shows the results of all tested algorithms for synthetic functions with a cost budget of $2500$.
From Fig.~\ref{fig:syn1}a, MT-PES can achieve a similar averaged IR with a much lower cost than PES, which implies that the BO performance can be accelerated by exploiting the binary auxiliary information of lower evaluation cost. MT-ES achieves lower averaged IR than PES with a cost less than $1000$ but unfortunately performs less well in the remaining BO iterations.
Even though the cheap auxiliary outputs provide additional information for finding the target maximizer at the beginning of BO, the multimodal nature of the synthetic function (see Fig.~\ref{fig:syn}a) causes MT-ES to be trapped easily in some local maximum since its search space has been pruned using EI for time efficiency.

To investigate how the performance of MT-PES will be affected by the proportion of positive outputs in different auxiliary functions,  we vary the number and bias $m_2$ of the auxiliary function(s) and show the results in Fig.~\ref{fig:syn1}b.
It can be observed that MT-PES using aux2 as the auxiliary function does not converge as fast as MT-PES using aux1, which is expected since aux2 with a larger proportion of positive outputs is less informative in 
delineating the input regions containing the target maximizer
than aux1. Also, Fig.~\ref{fig:syn1}b shows that MT-PES is able to exploit multiple auxiliary functions with different costs to achieve a lower averaged IR than PES with a much lower cost.

\emph{Remark.} From the results in Fig.~\ref{fig:syn1}b, one may expect MT-PES to converge faster using an auxiliary function with a smaller proportion of positive outputs, which is not always the case. If the auxiliary function has sparse positive outputs, MT-PES will face difficulty finding a positive output when exploring the auxiliary function and start to evaluate the target function after only several negative outputs are observed from evaluating the cheap auxiliary function. These negative outputs may not be informative enough to guide the algorithm to directly evaluate the target function near to the likely target maximizer.
%
To reduce the negative effect of such an unexpected behavior in real-world applications with an unknown auxiliary function, we can set MT-PES to evaluate only the auxiliary function using a small amount (e.g., $10\%$) of the budget at the beginning of BO so that positive auxiliary outputs are highly likely to be observed before MT-PES chooses to evaluate the expensive target function.

To provide more insight into the approximations of MT-PES, we follow the PES paper \citep{hernandez2014} and show the accuracy of the EP approximations (Section~\ref{approxPES}) compared to that of the ground truth constructed using the rejection sampling method.
To verify how sensitive the performance of MT-PES is to different settings, we have also evaluated the performance of the tested algorithms using synthetic functions with varying costs $\lambda_i$, random features dimension $m$, and sampling size $S$. The results are reported in Appendix~\ref{a_syn}.

{\bf Hartmann-6D function}. The original Hartmann-$6$D function is used as the target function and to construct the binary auxiliary function, as detailed in Appendix~\ref{a_bench}. 
Fig.~\ref{fig:syn1}c shows results of the tested algorithms with $10$ different initializations. 
It can be observed that MT-PES converges faster to a lower averaged IR than PES. However, MT-ES does not perform well for Hartmann-$6$D function which is difficult to optimize due to their multimodal nature (i.e., $1$ global maximum and $6$ local maxima) and large input domain. The former causes MT-ES to be trapped easily in some local maximum while the latter prohibits MT-ES from finely discretizing the input domain to remain computationally tractable.\vspace{-1mm}

\subsection{REAL-WORLD EXPERIMENTS}\label{experi:real}\vspace{-1mm}
The tested algorithms are next used in hyperparameter tuning of a ML model in an image classification task and policy search for reinforcement learning.

{\bf Convolutional neural network (CNN) with CIFAR-10 dataset.} 
The six CNN\footnote{We use the 
	example code of keras (i.e., cifar10\_cnn.py)
	and switch the optimizer in their code to SGD.} hyperparameters to be tuned in our experiments are the learning rate of SGD in the range of $[10^{-5}, 1]$, three dropout rates in the range of $[0, 1]$,  batch size in the range of $[100, 1000]$, and number of learning epochs in the range of $[100, 1000]$.
We use training and validation data of size $50000$ and $10000$, respectively. The unknown target function to be maximized is the validation accuracy evaluated by training the CNN with all the training data.
The auxiliary function is the decision made using the \emph{Bayesian optimal stopping} (BOS) mechanism in~\citep{dai2019,muller2007} by setting $0.5$ as a threshold of the validation accuracy. In particular, we train the same CNN model with a smaller fixed dataset of size $10000$ randomly selected from the original training data and apply the BOS after $20$ training epochs. The BOS will early-stop the training and return $1$ if it predicts that a final validation accuracy of $0.5$ can be achieved with a high probability, and $-1$ otherwise.\footnote{A description of BOS is provided in Appendix~\ref{a_bos}.} 
The real training time is not known and varies with different settings of hyperparameters. To simplify the setting of the evaluation costs, we use $\lambda_1(x) = 1$ and $\lambda_2(x) = 0.2\times(20/x_\text{epochs})$ where $x_\text{epochs}$ is the number of learning epochs in each selected hyperparameter setting.\footnote{We use $20\%$ of the training data for evaluating the auxiliary function and early-stop the training after around $20$ epochs.}
For this experiment, we additionally compare the tested algorithms with \emph{multi-fidelity GP-UCB} (MF-GP-UCB) \citep{kandasamy2016} that can only exploit \emph{continuous} auxiliary functions. The auxiliary function of MF-GP-UCB is the validation accuracy evaluated by training the same CNN with the same data used for the auxiliary function of MT-PES.\footnote{One may consider constructing the auxiliary function of MF-GP-UCB with an even smaller training dataset such that its cost is similar to that of the binary auxiliary function. However, for any smaller training dataset, we can always early-stop the training and achieve a much cheaper binary auxiliary function, as compared to the continuous auxiliary function of MF-GP-UCB constructed using the same dataset.}
The actual wall-clock time shown in the results includes the time of both CNN training and BO.
The validation accuracy $f_1(\tilde{x}_{*})$ is evaluated by training the CNN with $\tilde{x}_{*}$ for the tested algorithms.

{\bf Policy search for reinforcement learning (RL).} 
We apply the tested algorithms to the CartPole task from OpenAI Gym and use a linear policy consisting of $8$ parameters in the range of $[0, 1]$. This task is defined to be a success (i.e., reward of $1$) if the episode length reaches $200$, and a failure (reward of $-1$) otherwise. 
The target function to be maximized is the success rate averaged over $100$ episodes with random starting states. The auxiliary function is the reward of one episode with a fixed starting state $(0, 0, 0.02, 0.02)$. $\lambda_1(x) = 100$ and $\lambda_2(x) = 1$ are used in the experiments.
The success rate $f_1(\tilde{x}_{*})$ is evaluated by running the CartPole task with $\tilde{x}_{*}$ as the policy parameters over $100$ episodes for the tested algorithms.
\begin{figure}
	\centering
	\begin{tabular}{cc}
		\hspace{-2.5mm}\includegraphics[scale=0.215]{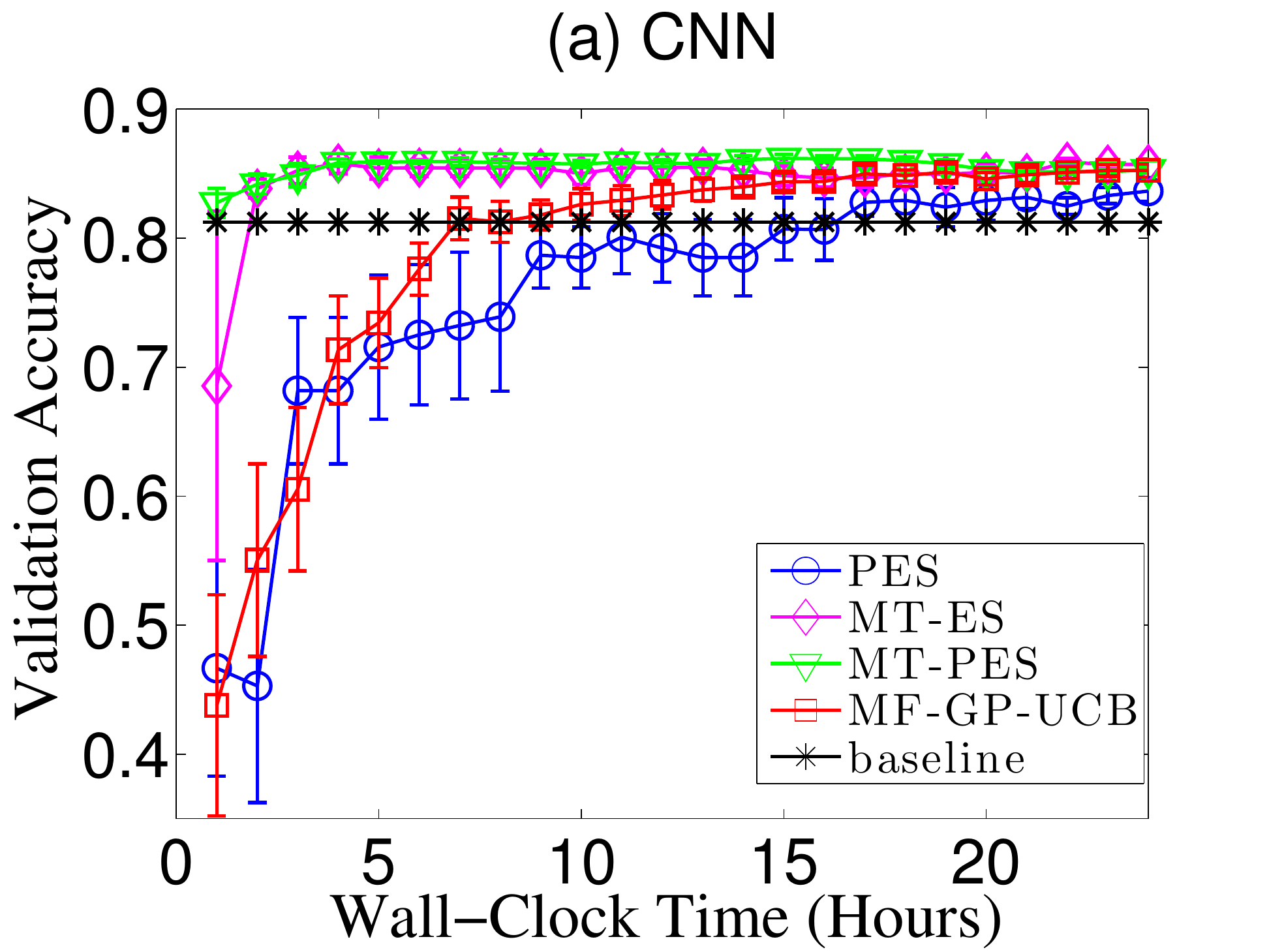} &
		\hspace{-7mm}\includegraphics[scale=0.215]{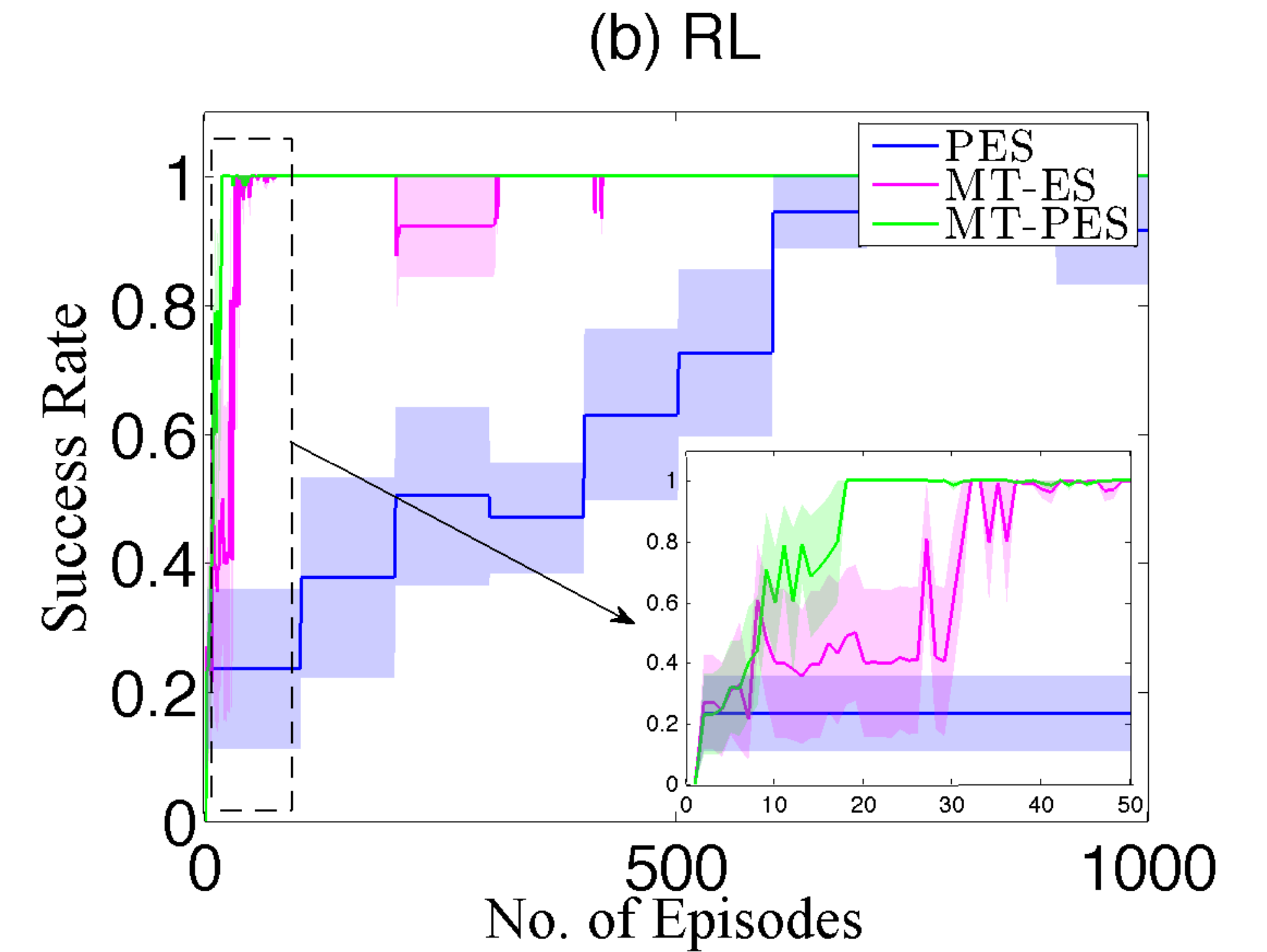} \vspace{-3mm}\\
	\end{tabular}
	\caption{Graphs of (a) validation accuracy vs.~wall-clock time incurred by tested algorithms for CNN and (b) success rate vs.~no. of episodes incurred by tested algorithms for RL. The results for the first 50 episodes are zoomed in for a clearer comparison.}\vspace{-3mm}
	\label{fig:real}
\end{figure}

Fig.~\ref{fig:real} shows results of the tested algorithms with $5$ different initializations for the CNN hyperparameter tuning and RL policy search tasks.
It can be observed that both MT-ES and MT-PES converge faster to a smaller IR than other tested algorithms. MT-PES also converges faster than MT-ES in both experiments.
MT-ES and MT-PES outperform MF-GP-UCB since evaluating the \emph{binary} auxiliary function by early-stopping the CNN training incurs much less time than evaluating the true validation accuracy for MF-GP-UCB.
Using only $1$ hour, MT-PES can improve the performance of CNN over that of the baseline achieved using the default hyperparameters in the existing code, which shows that MT-PES is promising in quickly finding more competitive hyperparameters of complex ML models.
\section{CONCLUSION}
This paper describes novel MT-ES and MT-PES algorithms for mixed-type BO that can exploit cheap binary auxiliary information for accelerating the optimization of a target objective function.
A novel mixed-type CMOGP model and its MT-RF approximation are proposed for improving the belief of the unknown target function and the global target maximizer using observations from evaluating the target and binary auxiliary functions.
New practical constraints are proposed to relate the global target maximizer to the binary auxiliary functions such that MT-PES can be approximated efficiently. Empirical evaluation on synthetic functions and real-world applications shows that MT-PES outperforms the state-of-the-art BO algorithms.
For future work, our proposed mixed-type BO algorithms can be easily extended to handle both binary and continuous auxiliary information, hence generalizing multi-fidelity PES~\citep{yehong17}.\footnote{A closely related counterpart is multi-fidelity active learning~\citep{YehongAAAI16}.}

{\bf Acknowledgements.}
This research is supported by the Singapore Ministry of Education Academic Research Fund Tier $2$, MOE$2016$-T$2$-$2$-$156$.

\bibliographystyle{natbib}
\bibliography{MTPES}

\appendix
\onecolumn

\section{RELATED WORK}\label{related}
Some existing BO works focus on optimizing a target function with a binary output type~\citep{gonzalez2017,tesch13} but have not considered utilizing the binary outputs for optimizing other correlated function which is more expensive to evaluate. The Bernoulli multi-armed bandit problem \citep{russo2018} assumes binary reward for each action and aims to maximize the cumulative rewards. However, the correlations between the arms and the cross-correlation between the \emph{immediate} binary reward and the averaged reward are ignored. Other than the multi-fidelity BO algorithms (Section~\ref{intro}), the constrained BO algorithms \citep{hernandez2016} also involve multiple functions (unknown target function and constraints) when optimizing the target function. Different from our mixed-type BO algorithms that can exploit the cross-correlation structure between the target and binary auxiliary functions, the constrained BO algorithms only consider continuous output types for the unknown constraints and assume the target and constraint functions to be independent.
Similar to our CNN experiment (Section~\ref{experi:real}), some hyperparameter optimization methods such as Hyperband~\citep{li2018} and BOHB~\citep{falkner2018} have considered speeding up their optimization process by early-stopping the training of underperforming models and continuing that of only the highly ranked ones.
However, both methods require the outputs (e.g., validation accuracy) to be continuous for ranking and do not consider the binary auxiliary information. Given the above idea, one may be tempted to exploit the binary information in a similar way: The binary auxiliary function is evaluated for a batch of inputs, and the target function is only evaluated at those inputs in the batch that yield positive auxiliary outputs for finding the global maximum. To achieve this, some important issues need to be considered:
(a) Which inputs should we select to evaluate the binary auxiliary function?
(b) How many binary auxiliary outputs should we sample before evaluating the expensive target function?
(c) If a large proportion of inputs in the batch yield positive auxiliary outputs, then evaluating the target function for all of them can also be very expensive. Which inputs should we select for evaluating the target function such that the global target maximizer can be found given a limited budget?
Our proposed MT-ES and MT-PES have resolved all the above issues in a principled manner.


\section{DERIVATION OF~\eqref{var}}\label{a_pred}

Since $f_1, \ldots, f_M$ are jointly modeled as a CMOGP, we know that
\begin{equation}\label{pfA}
p(f_A|f_{A'}) = \mathcal{N}(f_A|\mu_A+\Sigma_{AA'}\Sigma^{-1}_{A'A'}(f_{A'}-\mu_{A'}),\ \Sigma_{AA} - \Sigma_{AA'}\Sigma^{-1}_{A'A'}\Sigma_{A'A})
\end{equation} 
for any $A,A'\subseteq D^+$ \citep{Alvarez2011}.
Then,
\begin{equation}\label{qfX}
\begin{array}{rl}
q(f_X) \hspace{-2.8mm}&= q(f_{X_1}, f_{X_B}) \triangleq p(f_{X_1}|f_{X_B}) q(f_{X_B}) \vspace{1mm} \\
\hspace{-2.8mm}&\approx \mathcal{N}(f_{X}|\mu_{X}+\Sigma_{XX_B}(\Sigma_{X_BX_B} + \tilde{\Sigma}_B)^{-1}(\tilde{\mu} - \mu_{X_B}), \Sigma_{XX} - \Sigma_{XX_B}(\Sigma_{X_BX_B} + \tilde{\Sigma}_B)^{-1}\Sigma_{X_BX})
\end{array}
\end{equation}
due to \eqref{qfB}, \eqref{pfA}, and equation $9$c in \citep{schon2011}.
As a result, the posterior distribution $p(f_{X_1}, f_{X_B} |  y_{X_1}, y_{X_B})$ can be approximated with a multivariate Gaussian distribution:
\begin{equation}\label{pf}
\begin{array}{rcl}
p(f_{X_1}, f_{X_B} | y_{X_1}, y_{X_B}) &\hspace{-2.4mm} = &\hspace{-2.4mm}\displaystyle\frac{1}{Z} \ p(f_{X_1}|f_{X_B})\ p(y_{X_1} | f_{X_1})\ q(f_{X_B}) = \frac{1}{Z}\ p(y_{X_1} | f_{X_1})\ q(f_{X}) \vspace{1mm} \\
&\hspace{-2.4mm}\approx &\hspace{-2.4mm}\mathcal{N}(f_{X}|\mu_X + \Sigma_{XX}\Lambda^{-1}(\tilde{y}_X - \mu_X), \Sigma_{XX} - \Sigma_{XX}\Lambda^{-1}\Sigma_{XX})\ .
\end{array}
\end{equation}
The first equality is due to \eqref{pfX}. The last approximation is due to \eqref{qfX}, equation $10$f in \citep{schon2011}, and $p(y_{X_1}|f_{X_1}) \triangleq \mathcal{N}(y_{X_1}|f_{X_1}, \sigma_n) = \mathcal{N}(y_{X_1}|Mf_X, \Sigma_n)$ where $M \triangleq [I_{|X_1|\times|X_1|}, 0_{|X_1|\times|X_B|}]$.
Finally, the predictive belief in \eqref{var} can be obtained using \eqref{pfA}, \eqref{pf}, and equation $10$c in \citep{schon2011}.
\section{DETAILS OF MIXED-TYPE RANDOM FEATURES (MT-RF)}\label{a_conv}
Using some results of~\citet{rahimi2007}, the prior covariance of the GP modeling $L$ (Section~\ref{CMOGP}) can be rewritten as
\begin{equation}\label{Bochner}
\sigma_{xx'}  = \alpha\int p(w)\ e^{-jw^\top(x-x')}\ \text{d}w = 2\alpha\ \mathbb{E}_{p(w, b)}[\cos (w^\top x+b)\cos (w^\top x'+b)]
\end{equation}
where $p(w) \triangleq s(w)/\alpha$, $s(w)$ is the Fourier dual of $\sigma_{xx'}$, and $b \sim \mathcal{U}[0, 2\pi]$. Let $\phi(x)$ denote a random vector of an $m$-dimensional feature mapping of the input $x$:
\begin{equation} \label{phi}
\phi (x) \triangleq \sqrt{2\alpha /m}\ \cos (W^\top x+B)
\end{equation}
where $W \triangleq ({w_q})_{q=1,\ldots, m}$ and $B \triangleq ({b_q})^\top_{q=1,\ldots,m}$ with $w_q$ and $b_q$ sampled from $p(w)$ and $p(b)$, respectively. From~\eqref{Bochner} and~\eqref{phi}, the prior covariance $\sigma_{xx'}$ can be approximated by $\sigma_{xx'} \approx \phi(x)^\top \phi(x')$ and the latent function $L$ can be approximated by a linear model:
\begin{equation}\label{L}
L(x) \approx \phi (x)^\top \theta\ .
\end{equation}
Next, we will show how to derive the following approximation of $f_i(x)$:
\begin{equation}\label{conv}
f_i(x) \approx m_i + \phi_i(x)^\top \theta\ .
\end{equation}
\subsection{DERIVATION OF~\eqref{conv}}
Firstly, let $A$ be a $d \times d$ positive-definite diagonal matrix and $x$, $x'$, $w$, and $b$ be $d$-dimensional vectors. Then, the following convolutional result can be derived to be used in our derivation of~\eqref{conv}:
\begin{equation}\label{conv_exp}
\begin{array}{l}
\displaystyle \int_{x' \in D} e^{-\frac{1}{2}(x-x')^\top A(x-x')} e^{j(w^\top x'+b) }\ \text{d}x' \vspace{1mm}\\
\displaystyle= e^{jb}\int_{x' \in D} e^{-\frac{1}{2}(x^\top A x - 2x^\top A x' + x'^\top A x') +  jw^\top x'} \ \text{d}x' \vspace{1mm}\\
\displaystyle= e^{-\frac{1}{2}x^\top A x+jb}\int_{x' \in D} e^{-\frac{1}{2}x'^\top A x' + (x^\top A + jw^\top) x'} \ \text{d}x' \vspace{1mm}\\
\displaystyle= \sqrt{\frac{(2\pi)^d}{|A|}} e^{-\frac{1}{2}x^\top A x+jb} e^{\frac{1}{2}(x^\top A + jw^\top) A^{-1} (x^\top A + jw^\top)^\top} \vspace{1mm}\\
\displaystyle= \sqrt{\frac{(2\pi)^d}{|A|}} e^{-\frac{1}{2}x^\top A x+jb+\frac{1}{2}x^\top A x + jx^\top w - \frac{1}{2}w^\top A^{-1}w} \vspace{1mm}\\
\displaystyle= \sqrt{\frac{(2\pi)^d}{|A|}} e^{j(b+x^\top w) - \frac{1}{2}w^\top A^{-1}w}.
\end{array}
\end{equation}
The third equality follows from a result generalizing the Gaussian integral described at \url{https://en.wikipedia.org/wiki/Gaussian_integral#Generalizations}.

From~\eqref{func},
$$
\begin{array}{l}
f_i(x) \\
\displaystyle = m_i + \int_{x' \in D} K_i(x-x')\ L(x') \ \text{d}x' \vspace{1.5mm}\\ 
 \displaystyle \approx m_i + \int_{x' \in D} K_i(x-x')\ \phi (x')^\top \theta \ \text{d}x' \vspace{1mm}\\
\displaystyle = m_i + \sqrt{2\alpha /m} \times \theta^\top \left(\int_{x' \in D} K_i(x-x') \cos (w_q^\top x'+b_q) \ \text{d}x'\right)^\top_{q=1, \ldots, m} \vspace{1mm}\\
 \displaystyle = m_i + \sigma_{s_i} \sqrt{\frac{2\alpha}{m(2\pi)^d|P_i^{-1}|}} \times \theta^\top
\left(\int_{x' \in D} e^{-\frac{1}{2}(x-x')^\top P_i(x-x')} \cos (w_q^\top x'+b_q) \ \text{d}x'\right)^\top_{q=1, \ldots, m} \vspace{1mm}\\
\displaystyle = m_i + \sigma_{s_i} \sqrt{\frac{2\alpha}{m(2\pi)^d|P_i^{-1}|}} \times \theta^\top \left(\frac{1}{2}\int_{x' \in D} e^{-\frac{1}{2}(x-x')^\top P_i(x-x')}
\left(  e^{j(w_q^\top x'+b_q) } + e^{-j(w_q^\top x'+b_q)}\right) \ \text{d}x'\right)^\top_{q=1, \ldots, m} \vspace{1mm}\\
\displaystyle = m_i + \frac{1}{2}\sigma_{s_i} \sqrt{\frac{2\alpha}{m(2\pi)^d|P_i^{-1}|}} \times \sqrt{\frac{(2\pi)^d}{|P_i|}} \times \theta^\top \left(e^{j(b_q+x^\top w_q) - \frac{1}{2}w_q^\top P_i^{-1}w_q} + e^{-j(b_q+x^\top w_q) - \frac{1}{2}w_q^\top P_i^{-1}w_q}\right)^\top_{q=1, \ldots, m} \vspace{1mm}\\
 \displaystyle = m_i + \sigma_{s_i} \sqrt{\frac{2\alpha}{m}} \times \theta^\top \left( \frac{1}{2} e^{- \frac{1}{2}w_q^\top P_i^{-1}w_q}\left( e^{j(b_q+x^\top w_q)} + e^{-j(b_q+x^\top w_q)}\right)\right)^\top_{q=1, \ldots, m} \vspace{1mm}\\
 \displaystyle = m_i + \sigma_{s_i}\sqrt{2\alpha /m} \times \theta^\top \left(e^{-\frac{1}{2}w_q^\top P_i^{-1} w_q} \cos (w_q^\top x+b_q)\right)^\top_{q = 1, \ldots, m}\vspace{1mm} \\
 \displaystyle = m_i + \sigma_{s_i}\sqrt{2\alpha /m} \times \theta^\top \text{diag}(e^{-\frac{1}{2}W^\top P_i^{-1} W}) \cos (W^\top x+B) \vspace{1mm}\\
 \displaystyle = m_i + \sigma_{s_i}\theta^\top \text{diag}(e^{-\frac{1}{2}W^\top P_i^{-1} W})\phi(x)
\end{array}
$$
where $w_q$ is the $q$-th column of $W$ and $b_q$ is the $q$-th component of $B$.
The first approximation is due to~\eqref{L}. The second and last equalities follow from~\eqref{phi}. The third equality is due to the definition of the convolved kernel: $K_i(x) \triangleq \sigma_{s_i}\mathcal{N}( x|\underline{0}, P^{-1}_i)$. The fourth and third last equalities follow from the fact that $\cos(x) = \frac{1}{2}(e^{jx}+e^{-jx})$ which can be derived from the Euler's formula. The fifth equality is due to \eqref{conv_exp}. 

Then, let $\phi_i(x) \triangleq \sigma_{s_i}\ \text{diag}(e^{-\frac{1}{2}W^\top P_i^{-1} W})\ \phi(x)$. We can approximate $f_i(x)$ with $f_i(x) \approx m_i + \phi_i(x)^\top \theta$  and the approximated covariance $\sigma_{ij}(x, x') \approx \phi_i(x)^\top\phi_j(x')$ then characterizes the correlation within each function (i.e., $i = j$) and the cross-correlation between different functions (i.e., $i \neq j$).

\subsection{DERIVATION OF THE POSTERIOR DISTRIBUTION OF $\theta$} \label{a_conv2}

It follows from~\eqref{l} and~\eqref{conv} that $y_{X_i}$ is conditionally independent of $f_{X\setminus X_i}$, $W$, and $B$ given $f_{X_i}$ for $i=1,\ldots,M$ and $f_{X_1},\ldots,f_{X_M}$ are conditionally independent given $\theta$, $W$, and $B$, respectively. Then,
$$
p(y_X|\theta, W, B) = \prod_{i=1}^M \int p(y_{X_i}|f_{X_i})\ p(f_{X_i}|\theta, W, B) \ \text{d}f_{X_i}\ .
$$
From Section~\ref{CMOGP}, we know that $p(y_{X_1}|f_{X_1})$ is Gaussian and $p(y_i(x)|f_i(x))$ have been approximated as Gaussian using EP for $\langle x, i\rangle \in X_B$.
As a result, $p(y_X|\theta, W, B)$ can be approximated analytically as a multivariate Gaussian distribution and the posterior distribution of $\theta$ is\vspace{-1mm}
\begin{equation} \label{p3}
p(\theta|y_X) = \mathcal{N}(\theta|A^{-1}\Phi(\Lambda - \Sigma_{XX})^{-1}(\tilde{y}_X - \mu_X), A^{-1})
\end{equation}
where $\Phi \triangleq (\phi_j(x))_{\left\langle x, j \right\rangle \in X}$ and $A = \Phi (\Lambda - \Sigma_{XX})^{-1} \Phi^\top + I$.
\section{EP APPROXIMATION FOR~\eqref{approx1}}\label{A_EP1}
Let $t_1(f^*_1) \triangleq \Phi_\text{cdf}(({f_1(x_{*}) - y_{\text{max}}})/{\sigma_{n_1}})$ and $t_j(f^*_j) \triangleq \mathbb{I}(f^*_j + c_j \geq 0)$ for $j = 2, \ldots, M$. 
Then, $p(f^*|y_X, C2, C3)$ can be approximated by a multivariate Gaussian $q(f^*)$ such that each non-Gaussian factor is replaced by a Gaussian factor, that is, $t_j(f^*_j) \approx \tilde{t}_j(f^*_j) \triangleq \mathcal{N}(f_j^* | \tilde{\mu}_j, \tilde{\tau}_j)$ for $j = 1, \ldots, M$. Let $\tilde{\mu} \triangleq (\tilde{\mu}_j)^\top_{j=1,\ldots,M}$ and $\widetilde{\Sigma}$  be a $M \times M$ diagonal matrix with $\widetilde{\Sigma}_{jj} \triangleq \tilde{\tau}_j$ for $j=1,\ldots,M$. Then,
\begin{equation} 
p(f^*|y_X, C2, C3) = \frac{1}{Z} p(f^*|y_X) \prod_{j=1}^M t_j(f^*_j) \approx q(f^*) \triangleq \mathcal{N}(f^*|\mu, \Sigma) = \frac{1}{Z} \mathcal{N}(f^* | \mu_0, \Sigma_0) \prod_{j=1}^{M}\mathcal{N}(f_j^* | \tilde{\mu}_j, \tilde{\tau}_j)
\label{approx}
\end{equation}
where $\mu \triangleq \Sigma(\widetilde{\Sigma}^{-1}\tilde{\mu}+\Sigma_0^{-1}\mu_0)$ and $\Sigma \triangleq (\widetilde{\Sigma}^{-1} + \Sigma_0^{-1})^{-1}$ can be obtained using Gaussian identities, and $\mu_0$ and $\Sigma_0$ are, respectively, the posterior mean vector and covariance matrix of the Gaussian predictive belief $p(f^*|y_X)$ computed analytically using~\eqref{var}. 
With the multiplicative form of~\eqref{approx}, EP~\citep{minka2001} can be used to 
compute the Gaussian factors 
$\tilde{t}_j(f^*_j)=\mathcal{N}(f_j^* | \tilde{\mu}_j, \tilde{\tau}_j)$
for $j=1,\ldots,M$ in \eqref{approx}. Briefly speaking, EP will start from some initial values for $(\tilde{\mu}_j, \tilde{\tau}_j)$ and iteratively refine them, as shown in next subsection.

From~\eqref{approx}, the posterior distribution $p(f_i(x_{*})|y_X, C2)$ can be approximated by
\begin{equation}\label{fi}
\begin{array}{l}
\displaystyle p(f_i(x_{*})|y_X, C2) = \int p(f^*| y_X, C2) \ \text{d}f^*_1\ldots\text{d}f^*_{i-1}\text{d}f^*_{i+1}\ldots\text{d}f^*_{M} \\
\displaystyle \approx \int q(f^*) \  \text{d}f^*_1\ldots\text{d}f^*_{i-1}\text{d}f^*_{i+1}\ldots\text{d}f^*_{M} = \mathcal{N}(f_i(x_{*})|\mu_i, \tau_i)\vspace{-0mm}
\end{array}
\end{equation}
where 
$\mu_i$ is the $i$-th component of $\mu$ and $\tau_i$ is the $i$-th diagonal component of $\Sigma$.
\subsection{STEPS FOR EP APPROXIMATION}
EP is a procedure that starts from some initial values for the parameters $(\tilde{\mu}_j, \tilde{\tau}_j)$ of the Gaussian  factors $\tilde{t}_j(f^*_j)=\mathcal{N}(f_j^* | \tilde{\mu}_j, \tilde{\tau}_j)$ for $j = 1, ..., M$ and iteratively refines these quantities. At each iteration, for every Gaussian  factor $\tilde{t}_j(f^*_j)$, its contribution is removed to form the cavity distribution 
$$
q_{-j}(f^*) \propto q(f^*) / \tilde{t}_j(f_j^*) = \mathcal{N}(f^*|\mu_{-j}, \Sigma_{-j})\ .
$$
Then, the cavity distribution $q_{-j}(f_j^*)$ follows a Gaussian distribution $\mathcal{N}(f_j^*|\bar{\mu}_j, \bar{\tau}_j)$ with mean 
$\bar{\mu}_j \triangleq \bar{\tau}_j(\tau_j^{-1}\mu_j - \tilde{\tau}_j^{-1}\tilde{\mu}_j)$ and
variance $\bar{\tau}_j \triangleq (\tau_j^{-1} - \tilde{\tau}_j^{-1})^{-1}$.

Let $\hat{q}(f_j^*) \triangleq \mathcal{N}(f_j^*|\hat{\mu}_j, \hat{\tau}_j) \propto q_{-j}(f^*_{j})t_j(f^*_j)$ denote a new Gaussian distribution whose $j$-th Gaussian  factor $\tilde{t}_j(f_j^*)$ is replaced by its corresponding real factor $t_j(f^*_j)$. It is well-known that when $q(f^*)$ is Gaussian, the distribution that minimizes KL$(\hat{q}(f^*_j)||q(f^*_j))$ is one whose first and second moments match that of $\hat{q}(f_j^*)$. Let 
\begin{equation}\label{const}
\overline{Z}_j \triangleq \log \int \mathcal{N}(f_j^*|\bar{\mu}_j, \bar{\tau}_j)\ t_j(f^*_j) \ \text{d}f^*_j\ .
\end{equation}
Then, the moments can be updated to
\begin{equation}\label{update}
\hat{\mu}_j \triangleq \bar{\mu}_j + \bar{\tau}_j \frac{\partial \overline{Z}_j}{\partial \bar{\mu}_j}\quad\text{and}\quad \hat{\tau}_j \triangleq \bar{\tau}_j - \bar{\tau}^2_j \left( \left[ \frac{\partial \overline{Z}_j}{\partial \bar{\mu}_j}\right]^2 - 2\frac{\partial \overline{Z}_j}{\partial \bar{\tau}_j} \right).
\end{equation}
The parameters of the Gaussian  factor $\tilde{t}_j(f^*_j)=\mathcal{N}(f_j^* | \tilde{\mu}_j, \tilde{\tau}_j)$ can be computed with
\begin{equation}\label{update1}
\tilde{\mu}_j = \tilde{\tau}_j(\hat{\tau}_j^{-1}\hat{\mu}_j - \bar{\tau}_j^{-1}\bar{\mu}_j)\quad\text{and}\quad\tilde{\tau}_j = (\hat{\tau}_j^{-1} - \bar{\tau}_j^{-1})^{-1}\ .
\end{equation}
By applying the results in Appendix B.$2$ in~\citep{hernandez2014} to~\eqref{const},~\eqref{update}, and~\eqref{update1}, the parameters of  $\tilde{t}_1(f^*_1)$ can be refined to
$$
\tilde{\mu}_1 = \bar{\mu}_1 + \kappa_1^{-1}\quad\text{and} \quad \tilde{\tau}_1 = \beta_1^{-1} - \bar{\tau}_1
$$ 
where
$$\alpha_1 \triangleq \frac{\bar{\mu}_1 - y_{\text{max}}}{\sqrt{\bar{\tau}_1 + \sigma_{n_1}^2}},\
\beta_1 \triangleq \frac{\phi(\alpha_1)}{\Phi_\text{cdf}(\alpha_1)} \left[\frac{\phi(\alpha_1)}{\Phi_\text{cdf}(\alpha_1)} + \alpha_1 \right]\frac{1}{\bar{\tau}_1+\sigma_{n_1}^2}, \ 
\text{and} \ \kappa_1 \triangleq \left[\frac{\phi(\alpha_1)}{\Phi_\text{cdf}(\alpha_1)} + \alpha_1 \right]\frac{1}{\sqrt{\bar{\tau}_1+\sigma_{n_1}^2}}\ .$$

Next, we will describe how to update the parameters of $\tilde{t}_j(f^*_j)$ for $j = 2, \ldots, M$.
Due to \eqref{const},
\begin{equation}
\overline{Z}_j = \log \int \mathcal{N}(f_j^*|\bar{\mu}_j, \bar{\tau}_j)\ \mathbb{I}(f^*_j + c_j \geq 0) \ \text{d}f^*_j\ = \log \Phi_\text{cdf}(\frac{c_j + \bar{\mu}_j}{\sqrt{\bar{\tau}_j}}) \ .
\end{equation}
for $j = 2, \ldots, M$.
Then, the derivative of $\overline{Z}_j$ with respect to the posterior mean $\bar{\mu}_j$ and variance $\bar{\tau}_j$ can be computed as follows:
$$
\frac{\overline{Z}_j}{\partial \bar{\mu}_j} = \frac{\phi(\alpha_j)}{\Phi_\text{cdf}(\alpha_j)}\frac{1}{\sqrt{\bar{\tau}_j}} \quad \text{and} \quad \frac{\partial \overline{Z}_j}{\partial \bar{\tau}_j} = -\frac{\phi(\alpha_j)}{\Phi_\text{cdf}(\alpha_j)} \frac{c_j+\bar{\mu}_j}{2\bar{\tau}_j\sqrt{\bar{\tau}_j}}
$$
where $ \alpha_j \triangleq (c_j + \bar{\mu}_j)/\sqrt{\bar{\tau}_j}$.

Then, the moments can be updated using \eqref{update}:
\begin{equation}\label{update2}
\hat{\mu}_j \triangleq \bar{\mu}_j + \bar{\tau}_j \frac{\partial \overline{Z}_j}{\partial \bar{\mu}_j} = \bar{\mu}_j + \sqrt{\bar{\tau}_j} \frac{\phi(\alpha_j)}{\Phi_\text{cdf}(\alpha_j)}, \quad \hat{\tau}_j  \triangleq \bar{\tau}_j - \bar{\tau}^2_j \left( \left[ \frac{\partial \overline{Z}_j}{\partial \bar{\mu}_j}\right]^2 - 2\frac{\partial \overline{Z}_j}{\partial \bar{\tau}_j} \right) = \bar{\tau}_j - \bar{\tau}^2_j \beta_j
\end{equation}
where
$$
\beta_j \triangleq \frac{\phi(\alpha_j)}{\Phi_\text{cdf}(\alpha_j)}\left[ \frac{\phi(\alpha_j)}{\Phi_\text{cdf}(\alpha_j)} + \alpha_j \right]\frac{1}{\bar{\tau}_j}\ .
$$
Then, due to \eqref{update1} and \eqref{update2}, the parameters of $\tilde{t}_j(f^*_j)$ can be refined to 
$$
\tilde{\mu}_j = \bar{\mu}_j + \kappa_j^{-1}\quad\text{and} \quad \tilde{\tau}_j = \beta_j^{-1} - \bar{\tau}_j
$$ 
where
$$\kappa_j \triangleq \left[\frac{\phi(\alpha_j)}{\Phi_\text{cdf}(\alpha_j)} + \alpha_j \right]\frac{1}{\sqrt{\bar{\tau}_j}}$$
for $j = 2, \ldots,  M$.
\section{DERIVATION OF POSTERIOR DISTRIBUTION $p(f^+| y_X, C2, C3)$}\label{A_G1}
Let $X^\dagger \triangleq X \cup \{\langle x_{*}, i\rangle\}$. Then, 
\begin{equation} \label{G1}
p(f^+| y_X, C2, C3) = p(f_i(x)|y_X,f_i^*)\ p(f_i^*|y_X, C2, C3) = \mathcal{N}(f^+|\mu^+, \Sigma^+)
\end{equation}
with posterior mean vector $\mu^+\triangleq [\mu_i; \Psi[y_X;\mu_i]]$ and covariance matrix
$$
\Sigma^+ \triangleq \begin{bmatrix}
\tau_i & \tau_i \psi \\ 
\psi \tau_i &  \sigma^2_{\langle x, i\rangle| X^\dagger} + \psi^2 \tau_i
\end{bmatrix}
$$
where $\Psi \triangleq \Sigma_{\{\langle x, i\rangle\} X^\dagger}\Sigma^{-1}_{X^\dagger X^\dagger}$ and $\psi$ is the last component of $\Psi$.
Next, we will give the derivation of $\mu^+$ and $\Sigma^+$.

Firstly, the following lemma is needed:
\begin{lemma}\label{bo.lemma1}
	Let $a$, $b$, $c$ be three random vectors with dimension $n_a$, $n_b$, $n_c$ and
	$$
	p(a|c) = \mathcal{N}(a|\mu_a, \Sigma_a)
	$$
	$$
	p(b|a, c) = \mathcal{N}(b|\mu_{b|a,c}, \Sigma_{b|a,c})
	$$
	where $\mu_{b|a,c} \triangleq M_1a+M_2c+s = [M_1, M_2][a;c]+s$. Then, the conditional joint distribution of $a$ and $b$ given $c$ is
	$$
	p(a, b|c) = \mathcal{N}([a;b]|\mu_{a,b|c}, \Sigma_{a,b|c})
	$$
	where
$$
	\mu_{a,b|c} \triangleq \begin{bmatrix} \mu_a \\ [M_1, M_2][\mu_a; c]+s \end{bmatrix} \quad\text{and}\quad
	\Sigma_{a,b|c} \triangleq \begin{bmatrix}
	\Sigma_a & \Sigma_aM_1^\top\\ 
	M_1\Sigma_a &  \Sigma_{b|a,c} + M_1\Sigma_aM_1^\top
	\end{bmatrix}.
$$
\end{lemma}
\begin{proof}
	From the definition of multivariate Gaussian distribution,
	\begin{equation}\label{pab}
	p(a,b|c) = p(a|c)\ p(b|a,c) = \frac{(2\pi)^{-(n_a+n_b)/2}}{\sqrt{|\Sigma_{b|a,c}||\Sigma_a|}}e^{-\frac{1}{2}E}
	\end{equation}
	where
	$E \triangleq (b-\mu_{b|a,c})^\top \Sigma_{b|a,c}^{-1} (b-\mu_{b|a,c}) + (a-\mu_a)^\top \Sigma_a^{-1} (a-\mu_a)$.
	
	Let $f \triangleq b-M_1\mu_a-M_2 c - s$ and $e \triangleq a-\mu_a$. Then,
	\begin{equation}\label{E}
	\begin{array}{rcl}
	E &\hspace{-2.4mm} =&\hspace{-2.4mm} (b-M_1a-M_2c-s)^\top \Sigma_{b|a,c}^{-1} (b-M_1a-M_2c-s) + (a-\mu_a)^\top \Sigma_a^{-1} (a-\mu_a) \vspace{1mm}\\
	&\hspace{-2.4mm}= &\hspace{-2.4mm} (f - M_1e)^\top \Sigma_{b|a,c}^{-1} (f - M_1e) + e^\top \Sigma_a^{-1} e \vspace{1mm}\\
	&\hspace{-2.4mm}= &\hspace{-2.4mm} \begin{bmatrix} a-\mu_a\\ b-M_1\mu_a-M_2c-s \end{bmatrix}^\top R^{-1} \begin{bmatrix} a-\mu_a\\ b-M_1\mu_a-M_2c-s \end{bmatrix}
	\end{array}
	\end{equation}
	where
	$$
	R = \begin{bmatrix} M_1^\top \Sigma_{b|a,c}^{-1}M_1+\Sigma_a^{-1} & -M_1^\top \Sigma_{b|a,c}^{-1}\\ -\Sigma_{b|a,c}^{-1}M_1 &  \Sigma_{b|a,c}^{-1} \end{bmatrix}^{-1} = \begin{bmatrix} \Sigma_a & \Sigma_aM_1^\top\\ M_1\Sigma_a &  \Sigma_{b|a,c} + M_1\Sigma_aM_1^\top \end{bmatrix}.
	$$
	The last equality of \eqref{E} can be computed from equation $50$ in \citep{schon2011} and the second equality of $R$ is due to equation $9$d in \citep{schon2011}.
	Also, 
	$$
	\frac{1}{|R|} = \frac{1}{|\Sigma_a||\Sigma_{b|a,c}|}
	$$
	due to equation $51$ in \citep{schon2011}. Therefore,~\eqref{pab} can be written as
	\begin{equation}\label{pab1}
	\begin{array}{l}
	p(a,b|c) \vspace{1mm} \\
	\displaystyle = \frac{(2\pi)^{-(n_a+n_b)/2}}{\sqrt{|R|}} \exp \left(-\frac{1}{2} \begin{bmatrix} a-\mu_a\\ b-M_1\mu_a-M_2c-s \end{bmatrix}^\top R^{-1} \begin{bmatrix} a-\mu_a\\ b-M_1\mu_a-M_2c-s \end{bmatrix} \right) \vspace{1mm}\\
	\displaystyle = \mathcal{N}\left([a;b]\Bigg|\begin{bmatrix} \mu_a \\ [M_1, M_2][\mu_a; c]+s \end{bmatrix}, R\right).
	\end{array}
	\end{equation}
\end{proof}

Then, in~\eqref{G1}, we know that $p(f_i(x)|y_X, f_i^*) = \mathcal{N}(f_i(x)|\mu_{\langle x, i\rangle|X^\dagger}, \sigma^2_{\langle x, i\rangle|X^\dagger})$ with $\mu_{\langle x, i\rangle|X^\dagger} \triangleq \Sigma_{\{\langle x, i\rangle\} X^\dagger}\Sigma^{-1}_{X^\dagger X^\dagger} [y_X; f^*_i]$ and $p(f_i^*|y_X, C2, C3) = \mathcal{N}(f_i^*|\mu_i, \tau_i)$~\eqref{fi}. Therefore,~\eqref{G1} can be easily obtained  by replacing $a$, $b$, and $c$ in Lemma~\ref{bo.lemma1} with $f_i^*$, $f_i(x)$, and $y_X$, respectively.

\section{DERIVATION OF POSTERIOR COVARIANCE MATRIX IN~\eqref{EPr}}\label{A_EPr}
Let 
$r \triangleq a^\top f^+$.
From~\eqref{G1} and~\eqref{approx2},
\begin{equation}\label{Z'}
\begin{array}{rl}
 Z' \hspace{-2.8mm} &\displaystyle = \int \mathcal{N}(f^+|\mu^+, \Sigma^+)\ \mathbb{I}(f_i(x) - f_i(x_{*}) \leq \delta_i c_i) \ \text{d}f^+ \vspace{1mm}\\
\hspace{-2.8mm} &\displaystyle = \int \mathcal{N}(r|\eta, v)\ \mathbb{I}(r \leq \delta_i c_i) \ \text{d}r =  \Phi_\text{cdf}\left(\frac{\delta_i c_i - \eta}{\sqrt{v}}\right).
\end{array}
\end{equation}
Let $\overline{Z}'\triangleq \log Z'$.
Then, the derivative of $\overline{Z}'$ with respect to the posterior mean vector $\mu^+$ and covariance matrix $\Sigma^+$ can be computed as follows:
$$
\frac{\partial\overline{Z}'}{\partial \mu^+} = \frac{\partial \overline{Z}'}{\partial \eta} \frac{\partial \eta}{\partial \mu^+} = \frac{1}{\Phi_\text{cdf}((\delta_ic_i-\eta)/{\sqrt{v}})}\ \phi\left(\frac{\delta_ic_i-\eta}{\sqrt{v}}\right) \left(-\frac{1}{\sqrt{v}}\right) a = -\frac{\gamma}{\sqrt{v}}a \ ,
$$
$$
\frac{\partial\overline{Z}'}{\partial \Sigma^+} = \frac{\partial\overline{Z}'}{\partial v} \frac{\partial v}{\partial \Sigma^+} = \frac{1}{\Phi_\text{cdf}(({\delta_ic_i-\eta})/{\sqrt{v}})}\ \phi\left(\frac{\delta_ic_i-\eta}{\sqrt{v}}\right) \frac{\eta - \delta_ic_i}{2v\sqrt{v}} aa^\top = \frac{\gamma (\eta - \delta_ic_i)}{2v\sqrt{v}}aa^\top.
$$
Then,
$$
\mu_{f^+} = \mu^+ + \Sigma^+ \frac{\partial \overline{Z}'}{\partial \mu^+} = \mu^+ - \frac{\gamma}{\sqrt{v}}\Sigma^+a
$$
and
\begin{equation}
\begin{array}{rl}
\Sigma_{f^+} \hspace{-2.8mm} & \displaystyle = \Sigma^+ - \Sigma^+ \left(\left[\frac{\partial \overline{Z}'}{\partial \mu^+}\right]\left[\frac{\partial \overline{Z}'}{\partial \mu^+}\right]^\top - 2\frac{\partial \overline{Z}'}{\partial \Sigma^+}\right)\Sigma^+ \vspace{1mm}\\
\hspace{-2.8mm} & \displaystyle = \Sigma^+ - \Sigma^+ \left(\frac{\gamma^2}{v}aa^\top - \frac{\gamma (\eta-\delta_ic_i)}{v\sqrt{v}}aa^\top\right)\Sigma^+ \vspace{1mm}\\
\hspace{-2.8mm} & \displaystyle = \Sigma^+ - \frac{\gamma}{v}\left(\gamma-\frac{\eta-\delta_ic_i}{\sqrt{v}}\right)\Sigma^+aa^\top \Sigma^+.
\end{array}
\end{equation}
The first equality is due to~\eqref{update}.
\section{GENERALIZING TO MULTIPLE LATENT FUNCTIONS} \label{multi_L}
\subsection{CMOGP WITH MULTIPLE LATENT FUNCTIONS}
Let $\{L_q(x)\}_{q=1, ..., Q}$ denote a set of $Q$ independent latent functions. Then, CMOGP defines each $i$-th function $f_i$ as
\begin{equation}\label{func1}
f_i(x) \triangleq m_i + \sum^Q_{q=1}\int_{x'\in D}K_{iq}( x-x')\ L_q(x')\ \text{d}x'\ .
\end{equation}
Similar to CMOGP with only one latent function, the work of~\citet{Alvarez2011} has shown that if every $\{L_q( x)\}_{x\in D}$ is an independent GP for $q = 1, \ldots, Q$, then $\{f_i(x)\}_{\langle x,i\rangle \in D^+}$ is also a GP.
Specifically, let $\{L_q( x)\}_{x\in D}$ be a GP with prior covariance $\sigma^q_{xx'} \triangleq \mathcal{N}( x - x'| \underline{0}, \Gamma_q^{-1})$ and
$K_{iq}(x) \triangleq \sigma_{s_i q}\mathcal{N}( x|\underline{0}, P^{-1}_i)$. Then, 
\begin{equation}\label{kernel1}
\sigma_{ij}(x, x') = \sum^Q_{q=1} \sigma_{s_i q}\sigma_{s_j q}\mathcal{N}( x -  x'| \underline{0}, \Gamma_q^{-1}+P^{-1}_i+P^{-1}_j)\ .
\end{equation}
The Gaussian predictive belief in~\eqref{var} and the subsequent results in Section~\ref{MFPES} related to mixed-type CMOGP remain valid by computing its posterior covariance matrix with~\eqref{kernel1} instead of~\eqref{kernel}.
%
%
\subsection{MT-RF APPROXIMATION WITH MULTIPLE LATENT FUNCTIONS}
In this subsection, we will extend the MT-RF approximation described in Section \ref{MORF} to approximate the mixed-type CMOGP model with multiple latent functions.

Similar to that in Section~\ref{MORF}, the covariance function of the GP modeling  $L_q$ can be written as
$$
\begin{array}{rcl}
\sigma^q_{xx'} &\hspace{-2.4mm}=&\hspace{-2.4mm}\displaystyle \alpha_q\int p(w_q)\ e^{-jw_q^\top(x-x')} \ \text{d}w_q \vspace{1mm}\\
&\hspace{-2.4mm}\displaystyle= &\hspace{-2.4mm} 2\alpha_q \ \mathbb{E}_{p(w_q, b_q)}[\cos (w_q^\top x+b_q)\cos (w_q^\top x'+b_q)]
\end{array}
$$
where $p(w_q) \triangleq s(w_q)/\alpha_q$, $s(w_q)$ is the Fourier dual of $\sigma^q_{xx'}$, and $b_q \sim \mathcal{U}[0, 2\pi]$.

Then, each latent function $L_q$ can be approximated by a linear model:
\begin{equation}\label{L1}
L_q(x) \approx \phi_q (x)^\top \theta_q
\end{equation}
where $\phi_q (x) \triangleq \sqrt{2\alpha_q /m}\ \cos (W_q^\top x+B_q)$ for $q = 1, \ldots, Q$, and $W_q$ and $B_q$ consist of $m$ stacked samples from $p(w_q)$ and $p(b_q)$, respectively.

Let
\begin{equation}\label{fiq}
f_{iq}(x) \triangleq \int_{x'\in D}K_{iq}( x-x')\ L_q(x')\ \text{d}x' \ .
\end{equation}
Then,
\begin{equation}\label{fi1}
f_i(x) = m_i + \sum^Q_{q=1} f_{i q}(x) = m_i + \sum^Q_{q=1} \phi_{iq}(x)^\top\theta_q = m_i + \Phi_i(x)^\top\theta
\end{equation}
where $\theta \hspace{-0.5mm} \triangleq \hspace{-0.5mm} (\theta^\top_q)^\top_{q=1, \ldots, Q}$, $\Phi_i(x) \hspace{-0.5mm} \triangleq \hspace{-0.5mm}  (\phi_{iq}(x)^\top)^\top_{q=1, \ldots, Q} $, and $\phi_{i q}(x) \hspace{-0.5mm} \triangleq \hspace{-0.5mm} s\sigma_{s_i q}\ \text{diag}(e^{-\frac{1}{2}W_q^\top P_i^{-1} W_q})\ \phi_q(x)$ can be interpreted as the input features of function $f_i(x)$ corresponding to the latent function $L_q(x)$. The first equality is due to~\eqref{func1} and~\eqref{fiq}. The second equality is due to \eqref{conv}, \eqref{L1}, and \eqref{fiq}.

Since \eqref{fi1} has exactly the same form as \eqref{conv}, all the  results in Section~\ref{MORF} will remain valid for MT-RF approximation with multiple latent functions. 
\section{ADDITIONAL EXPERIMENTAL RESULTS}\label{a_experiment}
\subsection{SYNTHETIC FUNCTIONS}\label{a_syn}
The CMOGP hyperparameters for constructing the synthetic functions are fixed as follows: $\Gamma \triangleq \text{diag}[100, 100], P_1 \triangleq \text{diag}[2000, 100], P_2 \triangleq \text{diag}[100, 2000], \sigma_{s_1} \triangleq \sigma_{s_2} \triangleq 1$, $\sigma^2_{n_1} \triangleq 0.01$, and $m_1 = 0$.

To show the accuracy of the EP approximations for the constraints in Section~\ref{approxPES}, we compare the plot of EP approximations with that of the ground truth for \eqref{PES} using our synthetic functions. Similar to that in~\cite{hernandez2014}, we can construct the ground truth of \eqref{PES} using the \emph{rejection sampling} (RS) method since our synthetic functions are sufficiently simple. Examples of \eqref{PES} produced by RS and MT-PES using $5$ and $50$ observations from evaluating the target and aux1 functions are shown in Fig.~\ref{fig:acq}. As can be seen, the acquisition function achieved by the EP approximations is quite similar to the ground truth.
\begin{figure*}
	\centering
	\begin{tabular}{cc}
		\hspace{-0mm}\includegraphics[scale=0.28]{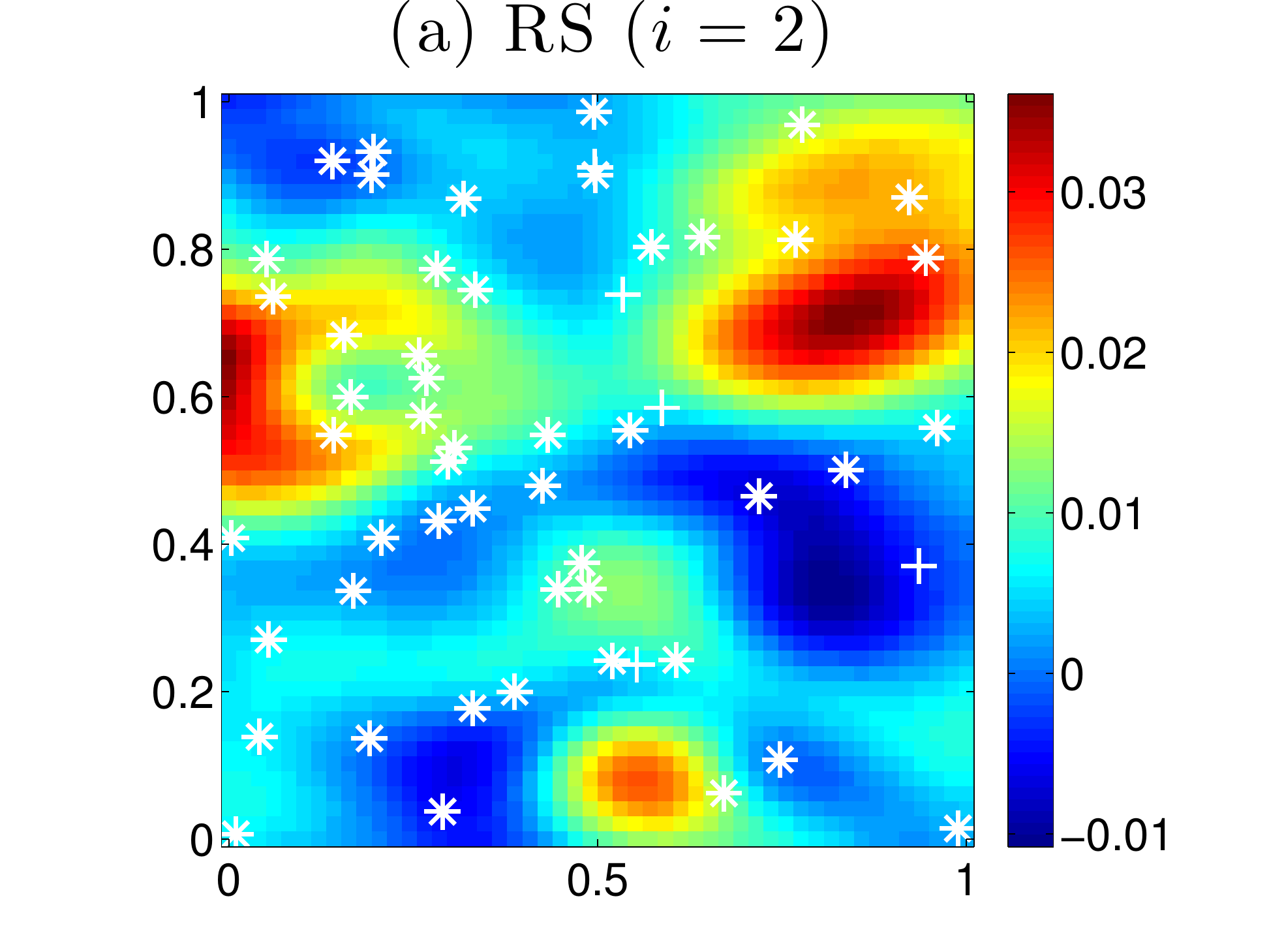} &
		\hspace{-0mm}\includegraphics[scale=0.28]{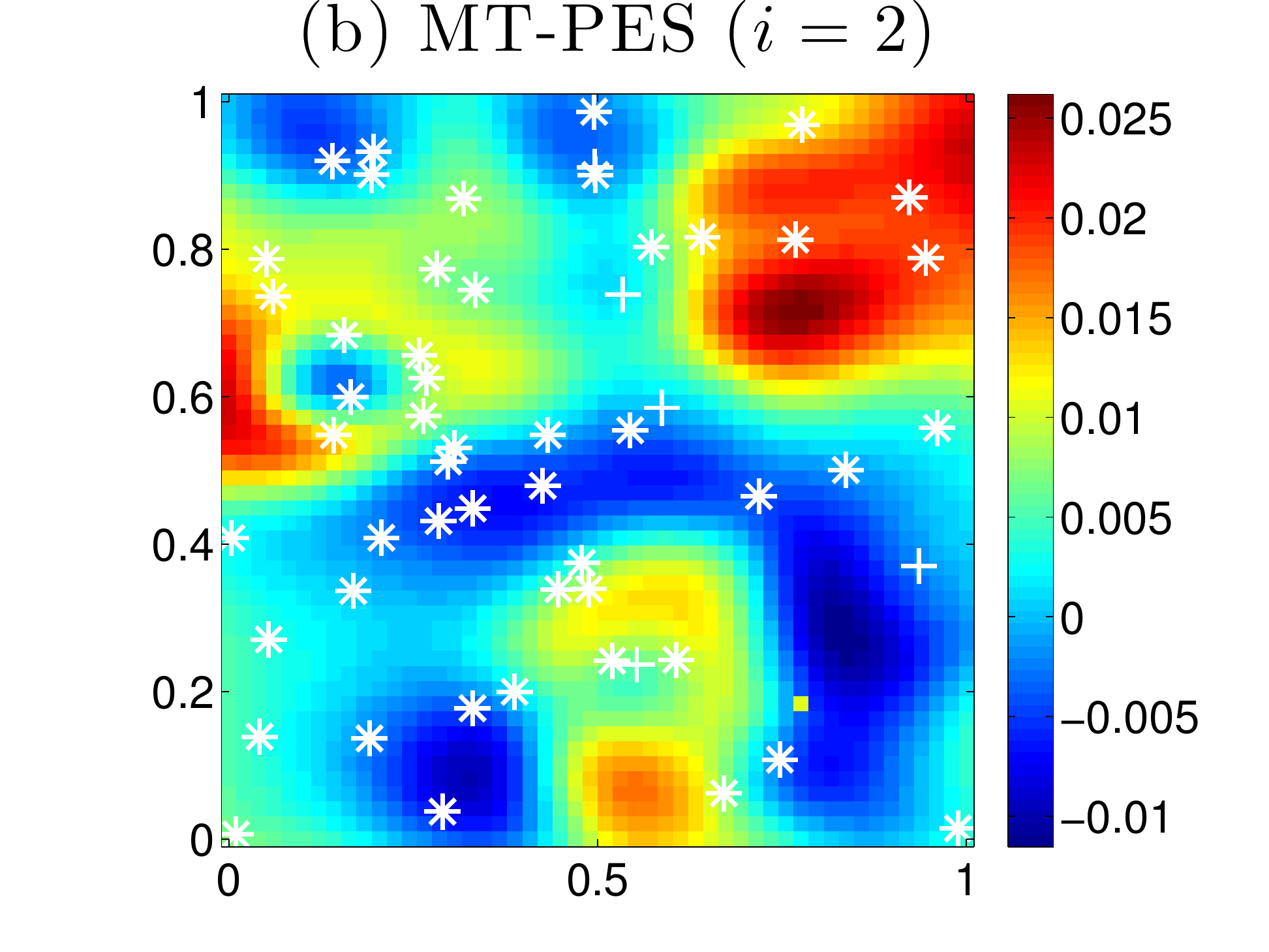}
		\vspace{-4mm} \\
	\end{tabular}
	\caption{Examples of the acquisition function \eqref{PES} with $i = 2$ obtained by (a) the \emph{rejection sampling} (RS) method and (b) our proposed MT-PES where `$+$' and `$*$' are inputs of the observations from evaluating the target and aux1 functions, respectively.}
	\label{fig:acq}
\end{figure*}

Results of MT-PES with varying costs, random features dimension, and sampling size are shown in Fig.~\ref{fig:syn_a}. It can be observed from Fig.~\ref{fig:syn_a}a that MT-PES converges faster than PES when the cost ratio of evaluating the target and auxiliary functions is larger than $25$.
Intuitively, MT-RF can achieve a more accurate approximation with a larger random feature dimension $m$ and sampling size $S$.
Figs.~\ref{fig:syn_a}b and~\ref{fig:syn_a}c show that the performance of MT-PES is robust to varying $S$ and decreases when $m$ is too small (i.e., $m=10$).
%

\begin{figure*}
	\centering
	\begin{tabular}{ccc}
		\hspace{-1mm}\includegraphics[scale=0.28]{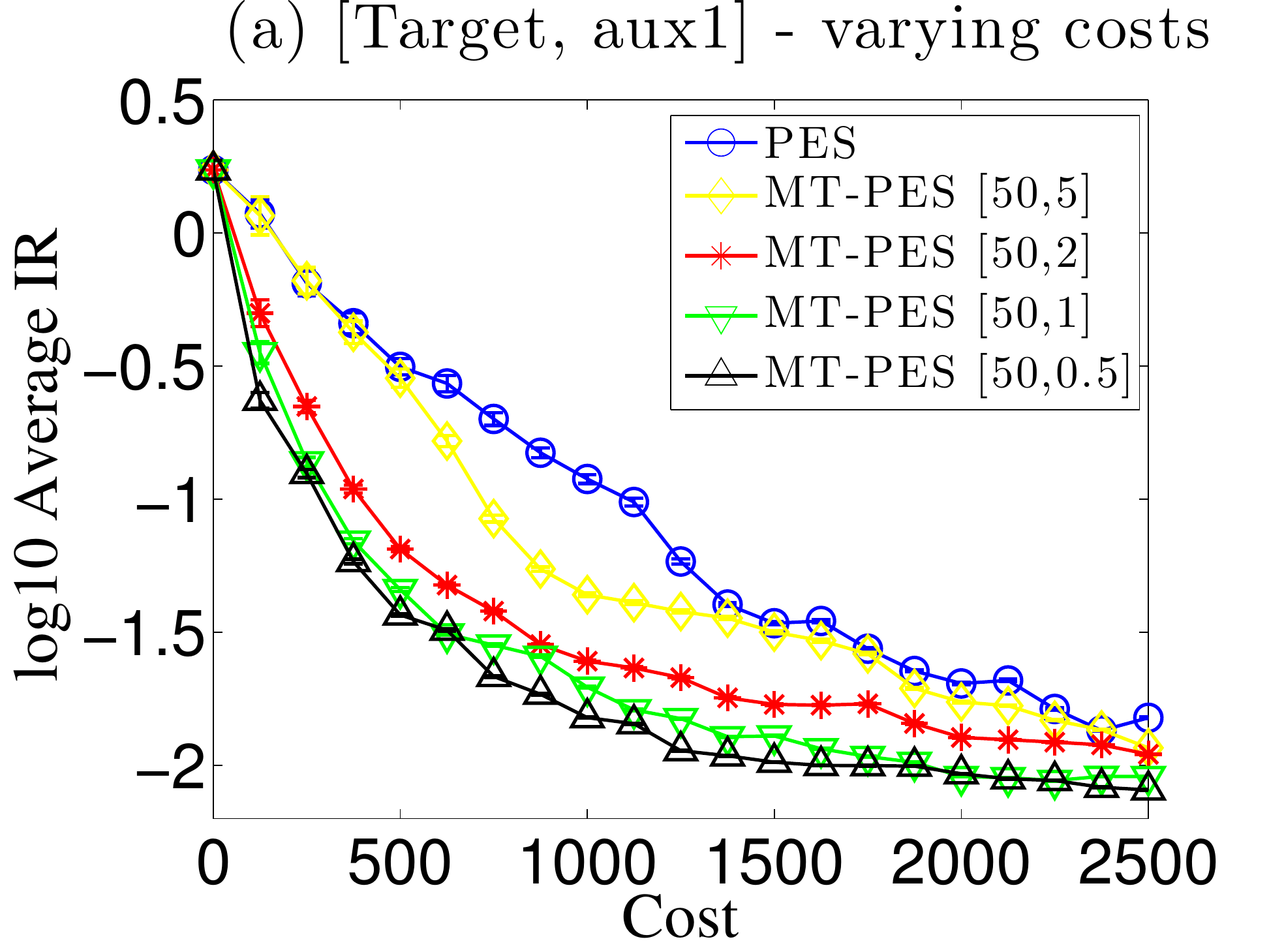} &
		\hspace{-4mm}\includegraphics[scale=0.28]{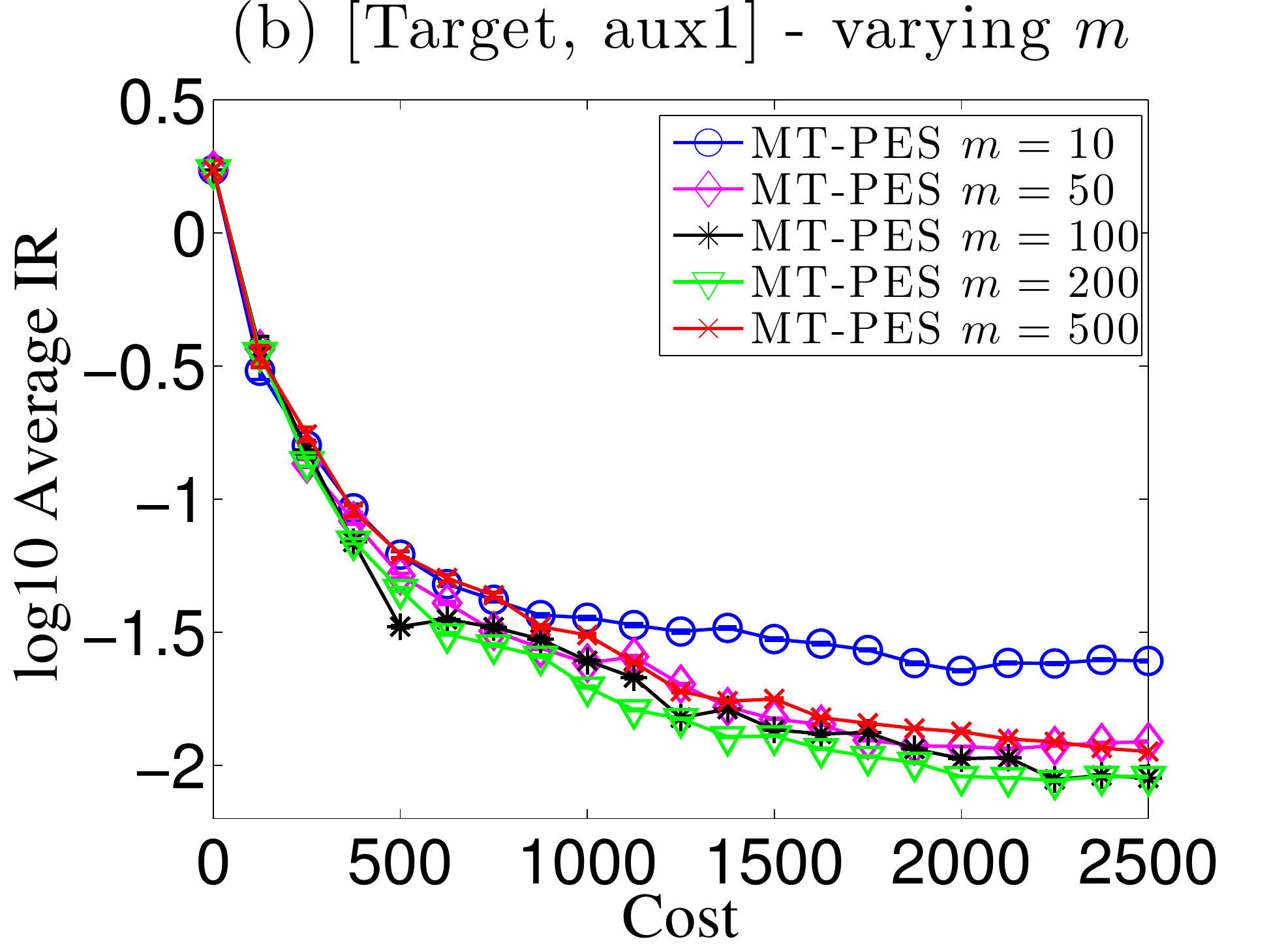} &
		\hspace{-4mm}\includegraphics[scale=0.28]{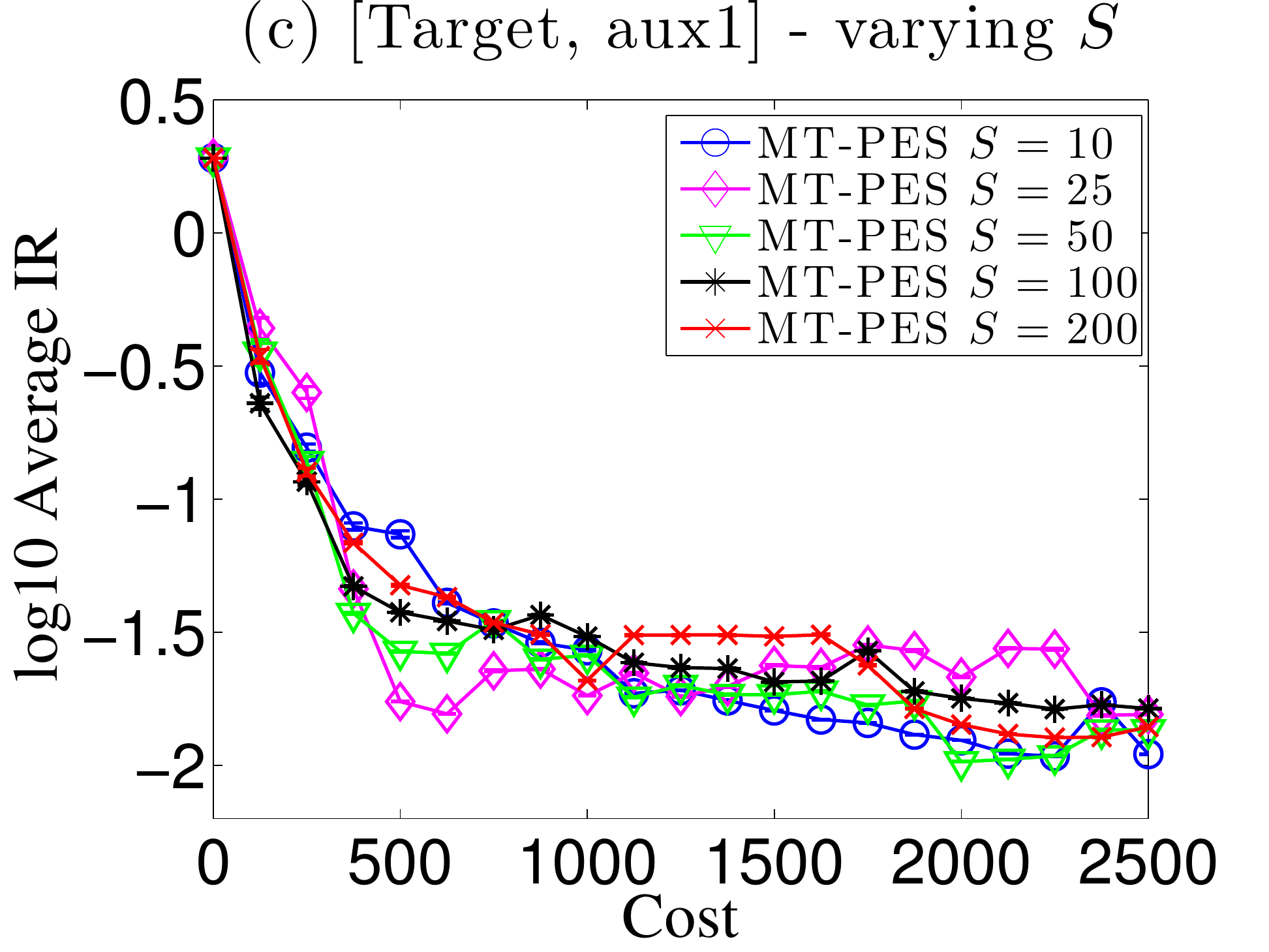}
		\vspace{-4mm} \\
	\end{tabular}
		\caption{Graphs of $\log_{10}(\text{averaged IR})$ vs. cost incurred by tested algorithms for the synthetic target and aux1 functions with (a) varying costs $\lambda_i$ for $i = 1$ and $2$, (b) varying random feature dimension $m$, and (c) varying sampling size $S$. The error bars are computed in the form of standard error.}
	\label{fig:syn_a}
\end{figure*}

\subsection{HARTMANN-6D FUNCTION}\label{a_bench}
Let $x_{(i)}$ be the $i$-th component of an input $x$. The following benchmark functions are used in our experiments: 

$D \triangleq [0, 1]^6$, $f_1(x) \triangleq \sum^4_{j=1} \beta_j \exp( \sum_{k=1}^6 A_{jk}(x_{(k)} - P_{jk})) - 0.2561$ where $A, P \in \mathbb{R}^{4 \times 6}$ are fixed matrices:

\begin{center}
$A \triangleq \left[ \begin{matrix}
10 \hspace{-1.5mm}& 3 \hspace{-1.5mm}& 17 \hspace{-1.5mm}& 3.5 & 1.7 & 8  \\ 
0.05 \hspace{-1.5mm}& 10 \hspace{-1.5mm}& 17 \hspace{-1.5mm}& 0.1 & 8 & 14  \\ 
3 \hspace{-1.5mm}& 3.5 \hspace{-1.5mm}& 1.7 \hspace{-1.5mm}& 10 & 17 & 8  \\ 
17 \hspace{-1.5mm}& 8 \hspace{-1.5mm}& 0.05 \hspace{-1.5mm}& 10 & 0.1 & 14
\end{matrix} \right]$, 
$P \triangleq 10^{-4} \times \left[ \begin{matrix}
1312 & 1696 & 5569 & 124 & 8283 & 5886  \\ 
2329 & 4135 & 8307 & 3736 & 1004 & 9991  \\ 
2348 & 1451 & 3522 & 2883 & 3047 & 6650  \\ 
4047 & 8828 & 8732 & 5743 & 1091 & 381
\end{matrix} \right]$
\end{center}

and $\beta_j$ is the $j$-{th} component of the vector $\beta \triangleq [1.0, 1.2, 3.0, 3.2].$ $y_1(x) \triangleq f_1(x) + \epsilon_1$ where $\epsilon_1 \sim \mathcal{N}(0, 10^{-3})$. $f_2(x) \triangleq f_1(x)$ and $y_2(x)$ is set to be $1$ if $f_2(x) \geq 0$, and $-1$ otherwise.

\subsection{DETAILS OF BAYESIAN OPTIMAL STOPPING IN CNN HYPERPARAMETER TUNING}\label{a_bos}

The training of a CNN under a given hyperparameter setting is an iterative process for some number $T$ of training epochs. After each training epoch $t = 1, \ldots, T$, the validation accuracy $v_t$ of the CNN trained thus far can be evaluated. As a result, a sequence of the validation accuracies (i.e., $v_1, \ldots, v_t$) can be obtained after $t$ training epochs and then used for predicting the final validation accuracy $v_T$.

Therefore, BOS models the training of the CNN as a sequential decision-making problem. After each training epoch, the BOS algorithm can choose from one of the three actions: $a_1$ = ``stop the training and conclude that $v_T \geq \delta$", $a_2$ = ``stop the training and conclude that $v_T < \delta$", and $a_3$ = ``continue to train for one more epoch" where $\delta$ is a performance threshold set as $0.5$ in our experiment. BOS maintains a posterior belief $p(v_T \geq \delta|v_1, \ldots, v_t)$ of the event “$v_T \geq \delta$” and choose the optimal action among $a_1$, $a_2$, and $a_3$ by minimizing an expected loss with respect to $p$ (see the algorithm in~\cite{muller2007} for details). If either $a_1$ or $a_2$ is taken, then the CNN training is early-stopped and the corresponding binary auxiliary output ($1$ for $a_1$ and $-1$ for $a_2$) is returned. Therefore, in principle, BOS early-stops the CNN training if it predicts that a final validation accuracy of $\delta$ can be achieved with a high probability and the binary decision is much cheaper since $t$ can be much smaller than $T$ when the CNN training is early-stopped.

\end{document}